\documentclass[conference,compsoc]{IEEEtran}

\usepackage{amsmath,amssymb,amsfonts,amsthm}
\usepackage{graphicx}
\usepackage[caption=false,font=footnotesize]{subfig}
\usepackage{multirow}
\usepackage{threeparttable}
\usepackage{cite}
\usepackage{xurl}
\usepackage{hyperref}
\usepackage{xcolor}
\usepackage{booktabs}
\usepackage{caption}
\usepackage{authblk}
\usepackage[many]{tcolorbox}

\newtheorem{assumption}{Assumption} 
\newtheorem{theorem}{Theorem}

\ifCLASSINFOpdf
\else
\fi

\begin{document}
\title{Have You Ever Seen Them? Entity-level Membership Inference through Interrogating Large Language Models}

\author[]{Yiran Zhu}
\author[]{Ziqi Yang}
\affil[]{Zhejiang University}

\maketitle

\begin{abstract}
Large Language Models (LLMs) raise growing concerns about privacy leakage and copyright compliance. Membership inference is a key tool for assessing such risks, but existing studies mainly focus on whether specific samples or sample-based data units are used for training. We argue that LLMs exhibit a human-memory-like behavior: an LLM may not memorize a specific sample verbatim, yet it can accumulate and reveal knowledge about a real-world entity from scattered mentions. This analogy motivates us to examine whether an LLM can be interrogated like a human interviewee to reveal its exposure to entity-related information. Motivated by this question, we propose entity-level membership inference, which determines whether information related to a target entity is used in LLM training. We study this task in the practical label-only black-box setting, where only generated texts are observable. We formalize the task under clue, input, and model constraints, establish the necessary and sufficient conditions for its feasibility, and instantiate five interrogation strategies based on this formalization. The strategies use limited entity clues to construct prompts, elicit entity-related responses, and infer membership from semantic features among the generated texts. We construct entity-level datasets and adapt state-of-the-art sample-level label-only methods to the entity-level setting as baselines. Experiments on person entities show that our methods achieve AUC up to 0.97 and bring gains of 6.0\%--17.5\% in Balanced Accuracy over the best adapted baseline. 
\end{abstract}

\IEEEpeerreviewmaketitle

\section{Introduction}

The capabilities of Large Language Models (LLMs) \cite{raiaan2024review} are typically built on pretraining over large-scale corpora \cite{touvron2023llama}, which raises concerns about privacy leakage \cite{radford2019language,carlini2021extracting} and copyright compliance \cite{ButterickSaveri2022,AuthorsGuild2023,TheVerge2023,NYT2023OpenAI}. This issue is especially challenging because many state-of-the-art LLMs are closed-source or only partially disclosed \cite{achiam2023gpt,touvron2023llama2}. 
Therefore, in the absence of training-set visibility, determining whether privacy-sensitive or copyrighted data is used to train the model has become one of the central challenges in auditing privacy leakage and copyright risks in LLMs. Membership inference provides a representative technical framework for studying this problem, aiming to infer whether the given data belongs to the training dataset of a target model \cite{shokri2017membership}.

Existing membership inference studies are largely \textit{sample-level}: they typically assume that the adversary has access to one or more samples. 
Sample-level membership inference can be broadly classified into the following categories.
Early studies mainly consider individual samples, aiming to determine whether a given sample is used to train the target model \cite{shokri2017membership, carlini2022membership, yeom2018privacy, choquette2021label, mattern2023membership}. Subsequent work extends this setting to user inference, which determines whether a collection of samples associated with a particular user is included in the model’s training data \cite{song2019auditing, kandpal2024user}. Other studies investigate dataset inference, where the goal is to infer whether a specific dataset is incorporated into the training process \cite{dziedzic2022dataset, maini2024llm}. Related work also studies whether the training data of generative models contain samples with a given property \cite{wang2024property}. Although these settings differ in their inference targets, they generally use samples as the carriers of information.

While sample-level membership inference has proven effective in auditing sample data exposure, LLMs exhibiting human-like capabilities drive us to rethink how privacy and copyright audits should be conducted under this new paradigm. Humans rarely acquire knowledge through the verbatim memorization of raw materials. Instead, we excel at abstracting rules and retaining high-level concepts. This raises a pivotal question: do LLMs, like humans, consolidate pre-training data into underlying concepts rather than merely retaining isolated training samples? Since ``concepts'' themselves lack tractable mathematical boundaries, we draw on the classical theory of concepts\footnote{``The classical theory holds that concepts are essentially definitions, each characterized by fixed rules determining to which entities the concept applies.'' See: \url{https://en.wikipedia.org/wiki/Concept}.} and make an initial attempt to use real-world entities (such as people, organizations, or events) as a concrete representation of concepts. Consequently, we pose a research question: \textit{can we interrogate an LLM like a human interviewee to reveal its exposure to information related to a specific entity?} To address this question, we propose \textbf{entity-level membership inference} under the most challenging label-only black-box setting, which aims to determine whether information related to a target entity has been used in LLM training.

This human-like memory paradigm (i.e., adopting entities rather than sample-level data units as inference targets) offers several advantages.
First, entities better align with human semantic cognition and legal compliance assessment. Privacy and copyright concerns inherently center on entities (e.g., individuals, organizations/institutions, creative works) rather than isolated texts.
Second, entity mentions are often short and can appear in multiple samples.
Entity-level inference can thus help mitigate the distribution-shift problem in sample-level inference, such as temporal shifts in textual distributions, as well as the near-member problem arising from paraphrased or semantically similar texts \cite{duan2024membership}. 
Finally, the entity captures actual knowledge-memorization behavior. For example, a private individual may be mentioned across multiple sources, each revealing only partial information such as affiliation, address, or family relations. 
Even if the LLM cannot recall a specific sample, it may learn and reproduce facts, relations, and descriptive patterns associated with the individual.

We present a formalization of this task and propose corresponding inference methods under clue, input, and model constraints. These constraints structurally mirror real-world interviewing scenarios where the investigator holds only incomplete, fragmented evidence and operates under limited responses. Grounded in the key intuition that the LLM exposed to an entity during training may reveal additional related information when prompted with specific information about that entity (\textit{related information leakage}), we theoretically establish the necessary and sufficient conditions for the feasibility and theoretical boundaries of this task. To robustly approximate the boundary, we eschew dynamic dialogue loops due to their lack of reproducibility, feature intractability, and compounding error risks. Instead, we distill six classic paradigms from investigative interviewing, instantiating an overarching prompt template and five independent interrogation strategies. We construct strategy-specific prompts using the target entity and its limited related clues to query the target LLM. To align with the specific elicitation logic of different strategies, the responses are paired via two pairing schemes to compute semantic similarity vectors, which are subsequently processed by a binary classifier to yield the final membership decision.

We construct dedicated entity-level membership inference datasets to evaluate our framework. Experiments show that our entity-level methods substantially outperform state-of-the-art sample-level label-only baselines (evaluated on samples constructed to encapsulate the target entities and their related clues). On person category entities, our five strategies surpass the strongest baseline by 6.0\%--17.5\% in Balanced Accuracy and 0.072--0.17 in AUC. On organization category entities, our strategies consistently maintain performance superiority, although all methods undergo performance degradation. Our framework has applicability across target LLMs of varying scales and architectures. Balanced Accuracy and AUC differences are kept within 5.0\% and 0.036, respectively, although TPR@1\%FPR remains model-dependent. Furthermore, we conduct an in-depth analysis of parameters that may affect inference performance. Finally, a realistic case study on the colossal LLM validates our real-world utility.

We summarize our main contributions as follows:

\begin{itemize}

\item{To the best of our knowledge, we are the first to propose and formalize entity-level membership inference. For this task, we introduce an interrogation-based framework that conceptualizes LLMs as human interviewees.
}

\item{We theoretically establish the necessary and sufficient conditions for the feasibility of entity-level inference and instantiate five interrogation strategies to construct strategy-specific prompts.
}

\item{We construct entity-level datasets to evaluate our methods. Experiments demonstrate the effectiveness of our methods.
Furthermore, our empirical analysis reveals that localized biographical memories are easier to audit than diffuse organizational knowledge.
}

\end{itemize}

\section{Background}

\subsection{Large Language Models}

LLMs acquire strong language modeling capabilities and substantial world knowledge through self-supervised pretraining on large-scale text corpora.
Mainstream LLMs (such as the GPT family~\cite{achiam2023gpt}, Gemini~\cite{team2023gemini}, LLaMA~\cite{touvron2023llama}, and Pythia~\cite{biderman2023pythia}) typically adopt a decoder-only Transformer architecture~\cite{wolf2020transformers}.
At inference time, an LLM generates text by decoding from the conditional distribution $p_{\theta}(t_i \mid t_{<i})$.
Decoding can be deterministic (e.g., greedy decoding or beam search) or can incorporate additional heuristics such as contrastive decoding~\cite{su2022contrastive}.
Alternatively, it can be stochastic, e.g., nucleus sampling, which samples from the smallest set of tokens whose cumulative probability mass exceeds a threshold~\cite{holtzmancurious}.
When sampling is used, the temperature parameter controls the sharpness of the distribution: lower temperatures yield more peaked distributions and thus more deterministic outputs, whereas higher temperatures flatten the distribution and increase randomness and diversity.

\subsection{Label-only Membership Inference}
\label{background_mia}

Since this paper studies membership inference under the most restrictive label-only black-box setting (where only the generated texts are observable), we focus on introducing existing label-only membership inference methods.

Early label-only membership inference studies show that access to logits or probabilities is not strictly necessary for inferring membership \cite{choquette2021label}. 
The key intuition is that models tend to behave more consistently on training samples than on non-training samples, even when only discrete output labels are available. This work establishes the feasibility of membership inference under restricted black-box access, but it focuses on traditional classification models.
Tang et al. \cite{tang2023assessing} study label-only membership inference on fine-tuning datasets of personalized language models. 
In contrast, pretrained LLMs are trained on massive and diverse corpora and often exhibit stronger generalization ability. As a result, perturbations at the token level may be too coarse to capture the tiny differences between member and non-member samples.

He et al. \cite{he2025towards} propose PETAL, which exploits the correlation between per-token semantic similarity and token probability to perform label-only membership inference on pretrained LLMs. PETAL first learns a linear mapping from semantic similarity to token probability using a surrogate model. It then queries the target model to obtain token-level semantic similarities for a target sample. Using the learned mapping, PETAL approximates the probabilities assigned by the target model to each token and computes an alternative perplexity score. Membership is then inferred by thresholding this score.

Lu et al. \cite{lu2026context} propose In-Context Probing (ICP) for membership inference on fine-tuned LLMs and provide a label-only variant. The main method uses probability-based black-box access. It constructs probing contexts and measures changes in the target sample’s log-likelihood. Its label-only variant replaces log-likelihood with semantic similarity, following the intuition of He et al. \cite{he2025towards}. It is experimentally compared against the PETAL method. ICP includes two probing strategies: reference-data-based (ICP-MIA-Ref) and self-perturbation probing (ICP-MIA-SP). The former retrieves semantically similar input-output pairs from public reference data, while the latter constructs probing contexts through generative or masking-based perturbations.

Overall, existing label-only membership inference methods improve practical applicability but remain sample-level. In contrast, we study entity-level label-only membership inference, which better aligns with real-world privacy and copyright concerns.

\subsection{Modern Public Interrogation Techniques}

In modern criminal justice and intelligence practice, the traditional confession-centered ``interrogation'' model is being replaced in many settings by ``investigative interviewing''. This newer approach prioritizes information gathering and the reliability of evidence.

The PEACE model \cite{shepherd2021investigative,bull2014investigative} is one of the best-known procedural frameworks. It establishes predictable communicative boundaries, encouraging a free narrative followed by targeted follow-up query verification.
The Scharff technique \cite{granhag2016scharff} is a cooperative intelligence-gathering approach that guides narratives by leaving deliberate information ``gaps'' and implying extensive prior knowledge to elicit detailed confirmations.
CI/ECI \cite{memon2010cognitive} are mainly used with witnesses and victims. Grounded in memory-retrieval and cue-dependency theories, this framework optimizes external retrieval cues and reporting contexts to maximize recall volume and detail.
SUE \cite{hartwig2014strategic} is a tactical framework that focuses on evidence disclosure.  It manipulates the timing and ordering of disclosing held external evidence to evaluate the consistency of the subject's statement.
For children and other vulnerable groups, the NICHD protocol \cite{lamb2007structured} provides a structured approach to forensic interviewing. It maximizes open-ended narrative and minimizes suggestive biases by strictly controlling question types.
Compared with these information-focused frameworks, the Reid technique \cite{inbau2013criminal} reflects a more confession-oriented tradition. It employs behavioral analysis, leading questions, and escalating pressure to test compliance and elicit target admissions.
The method is controversial. Critics argue that it can increase compliance and increase the risk of false confessions.

To infer whether a target LLM has seen the specified entity object, we borrow from well‑established human interrogation techniques used in real-world practice and adapt them to query the LLM.

\section{Problem Statement}

\subsection{Problem Definition}

Let $\mathcal{M}$ denote a target LLM trained on an unknown corpus $\mathcal{D}_{\text{train}}$, \(\mathcal{I}(e)\) denote the information related to a target entity \(e\), which may include textual mentions, factual descriptions, relations to other entities, or contextual descriptions associated with \(e\).
Given \(e\), our goal is to infer whether \(\mathcal{I}(e)\) is contained in $\mathcal{D}_{\text{train}}$. 
We define the entity-level membership label as:
\begin{equation}
y_e =
\begin{cases}
1, & \text{if } \mathcal{I}(e) \in \mathcal{D}_{\text{train}}, \\
0, & \text{otherwise}.
\end{cases}
\end{equation}

We consider the label-only black-box setting. 
Thus, the adversary can only interact with $\mathcal{M}$ through prompts and observe generated texts.
For each \(e\), the adversary is given a clue set:
$\mathcal{C}_e = \{c_1, c_2, \ldots, c_m\}$,
where each \(c_i\) is a piece of information associated with \(e\), such as another related entity, a contextual hint, or a partial description. 
The adversary then conducts a sequence of prompt-based interactions with the LLM and uses the generated responses to infer \(y_e\).

Inspired by real-world investigative interviewing, we identify three types of constraints that should be considered when designing entity-level membership inference methods: the clue, input, and model constraints.

\textbf{The Clue Constraint.}
For the target entity \(e\), the adversary can access only a limited number of clues related to \(e\). 
Given the clue set \(\mathcal{C}_e\), we assume that its size is bounded by a clue budget:
$|\mathcal{C}_e| \leq B_c$.
This constraint reflects the realistic situation where the adversary has only incomplete prior information about \(e\).

\textbf{The Input Constraint.}
The adversary's interaction with $\mathcal{M}$ is restricted to the prompt $p \in \mathcal{P}_{\mathrm{HI}}$, where \(\mathcal{P}_{\mathrm{HI}}\) denotes the space of human-interpretable textual prompts.
Since LLMs are trained primarily on human-generated content (such as natural language and code), human-interpretable prompts may more effectively elicit entity-related semantic associations. In addition, this constraint keeps the query process consistent with the principles of human interviews.

\textbf{Model Constraints.}
The adversary's operational space is limited by the access interface and usage conditions of $\mathcal{M}$. 
In the label-only setting, the adversary can only observe generated text responses $r = \mathcal{M}(p)$. 
The interaction is also constrained by the model's supported capabilities (e.g., context length, format adherence) and practical interaction costs (e.g., budget, latency, rate limits).

We denote the admissible set of interrogation strategies under these constraints as
$\Pi_{\mathrm{adm}}$.
Any feasible interrogation strategy must satisfy
$\pi \in \Pi_{\mathrm{adm}}$.

\subsection{Motivation}

The key intuition behind our framework is that an LLM may reveal entity-related information beyond what is explicitly provided in the prompt when it has been exposed to such information during training. 
This is analogous to human investigative interviewing, where carefully designed questions may lead an interviewee to reveal additional relevant information that reflects their prior knowledge.

We refer to this intuition as related information leakage: under suitable prompt-based interactions, the model may generate additional information associated with \(e\) (e.g., related entities, factual descriptions, relations, or contextual details), even when such information is not directly provided in the prompt. We use $\Delta(e, \mathcal{C}_e)$ to denote the entity-related information elicited from the model, which is observable in the model's generated text. 
These observable traces can later be used to design statistics or decision rules (e.g., output stability under perturbation, response consistency, behavioral differences relative to the public reference model) for entity-level membership inference under label-only access.

\section{Method}

\subsection{Overview}

\begin{figure}[!t]
\centering
\includegraphics[width=1.0\columnwidth]{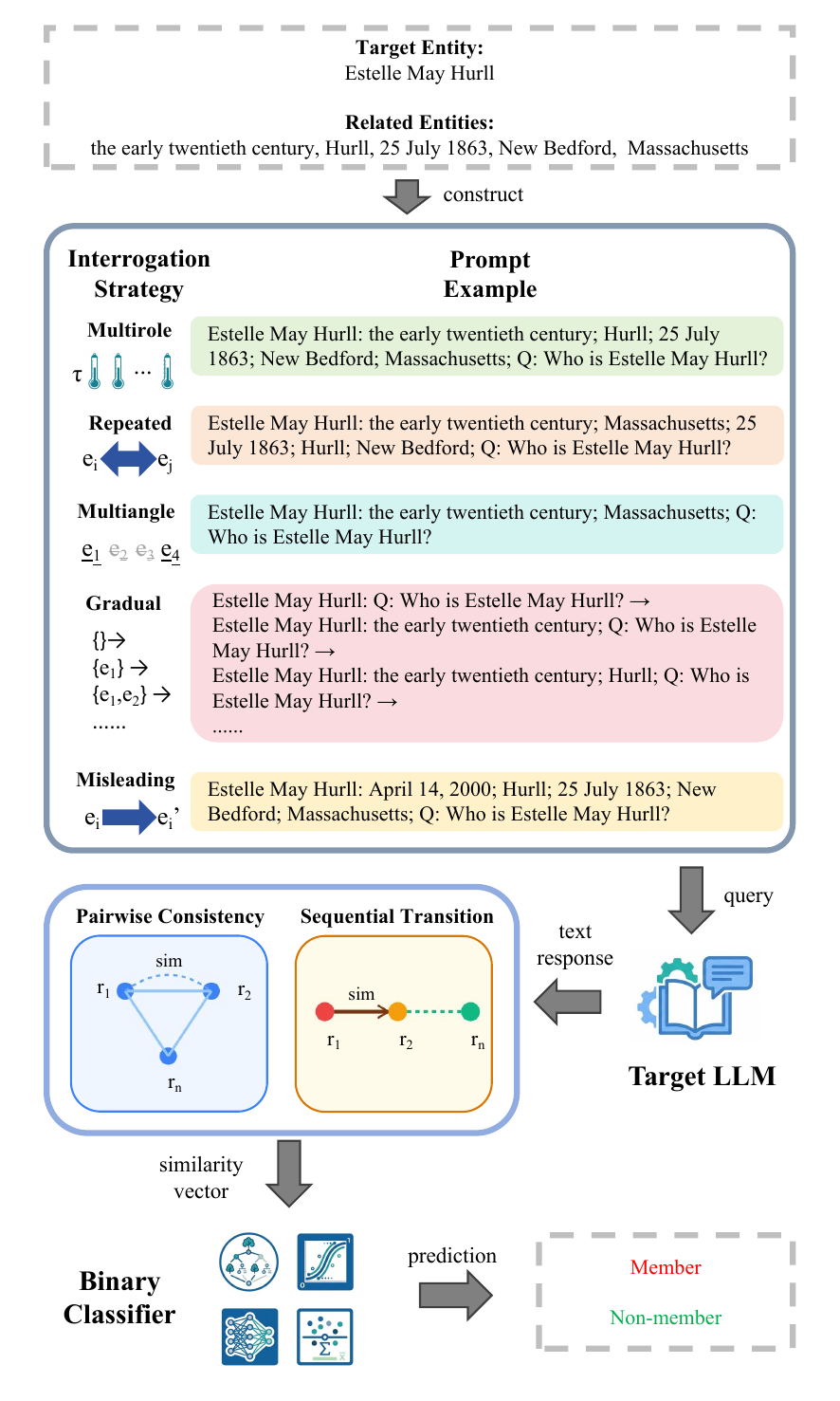}
\caption{Overview of entity-level membership inference.}
\label{overview}
\end{figure}

The overview of our entity-level membership inference process is shown in Fig. \ref{overview}.
Guided by the theoretical boundary and classic paradigms from investigative interviewing, we construct strategy-specific prompts using a target entity and its limited related clues to query the target LLM. 
Based on the intuition of related information leakage, we map elicited response sequences into semantic similarity vectors via two pairing schemes. These vectors are subsequently processed by the binary classifier to yield the final membership decision.

\subsection{Theoretical Analysis}

Before designing the entity-level membership inference method, we first theoretically establish the necessary and sufficient conditions for the feasibility of this task. 

\begin{assumption}[Hermetic Parameterization and Factual Non-Hallucination]
\label{ass:hermetic_non_hallucination}
We assume that the target model $\mathcal{M}$, parameterized by $\theta$, operates under a strictly closed knowledge boundary. It does not utilize external retrieval mechanisms, tool/API executions, test-time memory updates, or any other form of knowledge injection, meaning that under a fixed architecture, training objective, and optimization process, any factual knowledge encoded in $\theta$ derives solely from the training corpus $\mathcal{D}_{\text{train}}$. We further assume the principle of factual non-hallucination for out-of-distribution entities. For any target entity $e$, if $I(e) \cap \mathcal{D}_{\text{train}} = \emptyset$, the probability of $\mathcal{M}$ generating facts regarding $e$ that are not explicitly provided in the query prompt is zero, thereby preventing the model from consistently synthesizing accurate factual configurations of unseen entities by pure chance.
\end{assumption}
Under Assumption \ref{ass:hermetic_non_hallucination}, we present Theorem \ref{thm:necessary} as follows:

\begin{theorem}[Necessary Condition]
\label{thm:necessary}
If $\mathcal{I}(e) \notin \mathcal{D}_{\text{train}}$ (i.e., the information related to the target entity $e$ never appeared in the LLM's training data), then there exists no interrogation strategy $\pi$ satisfying $\pi \in \Pi_{\mathrm{adm}}$ that can induce the model to output information related to $e$. Formally:
\begin{equation}
\forall \pi\; (\pi \in \Pi_{\mathrm{adm}}) \; \Longrightarrow \; \Delta(e, \mathcal{C}_e) \not\in \mathcal{M}(\pi(e, \mathcal{C}_e)) 
\end{equation}
\end{theorem}

The proof of Theorem \ref{thm:necessary} is in Appendix \ref{proof}.
We observe that Theorem 1 is necessary but not sufficient: even if $\mathcal{I}(e) \in \mathcal{D}_{\text{train}}$, the adversary still needs to know an interrogation strategy $\pi$ from $e$ to $p_t$ that can reliably extract $\mathcal{I}(e)$. A sufficient condition (ideal scenario) is the adversary’s ability to arbitrarily control the prompt space, i.e., to construct a perfect prompt for any internal knowledge.

\begin{theorem}[Sufficient Condition]
\label{thm:sufficient}
Assume there exists a ``perfect prompt mapping function'' $\mathcal{F}$ satisfying:
\begin{enumerate}
\item For any internal entity-related information $\mathcal{I}(e) \in \mathcal{D}_{\text{train}}$, $p^* = \mathcal{F}(e, \mathcal{C}_e)$ satisfies $p^* \in \mathcal{P}_{\mathrm{HI}}$;
\item $\mathcal{M}(p^*) = \Delta(e, \mathcal{C}_e) = \mathcal{I}(e)$, i.e., upon receiving $p$, $\mathcal{M}$ outputs information $\Delta(e, \mathcal{C}_e)$ exactly matching $\mathcal{I}(e)$. 
\end{enumerate}
Then for any entity-related information $\mathcal{I}(e)$ (as long as $\mathcal{I}(e) \in \mathcal{D}_{\text{train}}$), there exists an interrogation strategy $\pi$ such that $\Delta(e, \mathcal{C}_e) \in \mathcal{M}(\pi(e, \mathcal{C}_e))$.
\end{theorem}

The proof of Theorem \ref{thm:sufficient} is in Appendix \ref{proof}.
This theorem describes the \textbf{theoretical upper bound} on what can be achieved: if there existed a way to losslessly map any internal knowledge to a safe prompt, then any information could be extracted. In practice, such a perfect mapping $\mathcal{F}$ is extremely difficult (or impossible) to construct, because there is a non-linear, non-invertible gap between the model's internal representations and human-interpretable texts.
Consequently, the central goal of practical interrogation is to \textbf{approximate} this ideal limit. Through multiple and indirect questioning, piece together and infer the target information from the model's outputs.

\subsection{Interrogation-specific Prompt Construction}
\label{prompt_construction}

To approximate the theoretical upper bound of interrogation, one might intuitively propose a dynamic conversational feedback loop similar to human interviews. However, open-ended conversational loops introduce severe methodological challenges. (1) Lack of reproducibility, as dynamic intermediate responses generate massive non-deterministic branching across models. (2) Feature intractability, as dynamic conversations yield variable-length, highly unstructured dialogue trajectories that are difficult to align. (3) Semantic drift and compounding errors, where a single early incorrect generation or hallucination permanently derails the subsequent questioning path. Consequently, to robustly approximate the theoretical elicitation threshold, we eschew active dialogue and instead propose a standardized, predefined static prompting scheme.

To systematically construct this interrogation framework, we refer to the rich literature of investigative interviewing and distill six classic paradigms: the PEACE model \cite{shepherd2021investigative}, the Cognitive Interview and Enhanced Cognitive Interview (CI/ECI) \cite{memon2010cognitive}, the Strategic Use of Evidence (SUE) technique \cite{hartwig2014strategic}, the NICHD Protocol \cite{lamb2007structured}, the Scharff technique \cite{granhag2016scharff}, and the Reid technique \cite{inbau2013criminal}. We leverage the Scharff technique as our overarching prompt template construction guideline, while instantiating the remaining five specialized techniques into our five distinct interrogation strategies (i.e., \textit{multirole}, \textit{repeated}, \textit{multiangle}, \textit{gradual}, and \textit{misleading} interrogation).

All five strategies follow the same clue and input constraints: given a target entity, the adversary uses only a limited set of related entities as clues and formulates human-interpretable prompts. They also share the same label-only model constraint, where the adversary observes only generated texts. The only optional capability we consider is whether the adversary is allowed to adjust the generation temperature, a commonly exposed hyperparameter that controls output randomness and creativity in many commercial LLM APIs \cite{hu2025membership}. In contrast, the default text generation mode is greedy decoding, which selects only the tokens with the highest probability.
In addition, inspired by the questioning principle in the Scharff technique, where blanks are deliberately left for the informant to fill in \cite{granhag2016scharff}, we construct prompts using only the given information (i.e., the target entity $e$ and its limited related entities $\mathcal{C}^e_e = \{c^e_1, c^e_2, \ldots, c^e_m\}$) rather than introducing additional external knowledge. Therefore, the initial prompt $p_0$ is constructed as: ``$e$: $c^e_1$; $c^e_2$; $\ldots$; $c^e_m$; Q: Who/What is $e$?''. To establish a stable benchmark for subsequent comparisons, we define the global baseline response $r_{\text{base}}$ generated under the default greedy decoding setting ($\tau = 0$) as:
\begin{equation}
    r_{\text{base}} = \mathcal{M}(p_0, \tau = 0)
\end{equation}

\textbf{Multirole interrogation} queries the target LLM as if it were multiple witnesses with different response styles. The investigative interviewing guidelines based on the PEACE model suggest that, when multiple witnesses are present, they should be separated immediately and interviewed individually \cite{newzealandpolice, brainscape}. Inspired by this principle, we simulate different ``witnesses'' by adjusting the temperature $\tau$ of the target LLM, which induces different generations for the same prompt. Formally, the adversary constructs a set of $n$ prompt–temperature pairs $\{ (p_0, \tau_i) \}_{i=1}^{n}$ with distinct temperatures ($\tau_i \neq \tau_j$ for $i \neq j$).
This corresponds to the least restrictive variant of our label-only setting: the adversary still observes only the generated texts, but is additionally allowed to control the temperature during generation.

\textbf{Repeated interrogation} repeatedly questions the target LLM with the same set of clues while varying their order. In CI/ECI, techniques such as change order and reverse order encourage interviewees to recall events \cite{memon2010cognitive}. Similarly, the SUE framework suggests that the timing and order of evidence disclosure can influence an interviewee's statement \cite{hartwig2014strategic}. Motivated by these observations, we construct 
a set of $n$ prompts $\{p_i\}_{i=1}^{n}$ using distinct permutations of the related entity sequence (denoted as $s_i$ where $s_i \neq s_j$ for $i \neq j$), thereby obtaining different generated texts.

\textbf{Multiangle interrogation} aims to elicit responses from the target LLM under different contextual angles. CI/ECI emphasizes changing perspectives \cite{memon2010cognitive}, while SUE encourages examining the consistency of statements across different evidence-based perspectives \cite{hartwig2014strategic}. Following these principles, we construct a set of $n$ prompts $\{p_i\}_{i=1}^{n}$ using randomly sampled subsets of the related entity set $\mathcal{C}^e_e$ (denoted as $\mathcal{C}^e_i \subseteq \mathcal{C}^e_e$ where $\mathcal{C}^e_i \neq \mathcal{C}^e_j$ for $i \neq j$).
Each subset provides a distinct contextual angle for interrogating the target LLM. 

\textbf{Gradual interrogation} progressively reveals clues to guide the target LLM. This design is inspired by SUE's strategic or late disclosure of evidence \cite{hartwig2014strategic} and the PEACE model's practice of eliciting an initial free narrative before using follow-up questions for detail elicitation and consistency checking \cite{shepherd2021investigative}. We initially construct a prompt using only the target entity, and then gradually append related entities to the prompt. This produces a sequence of prompts with increasing amounts of information.
Formally, the adversary constructs a sequence of prompts $\{p_i\}_{i=1}^{|\mathcal{C}^e_e|+1}$ by incrementally appending related entities to the target entity $e$ (denoted as $\mathcal{C}^e_i \subset \mathcal{C}^e_j$ for $1 \le i < j \le |\mathcal{C}^e_e|+1$, where $\mathcal{C}^e_i$ represents the subset of appended related entities in $p_i$, and $\mathcal{C}^e_1 = \emptyset$).

\textbf{Misleading interrogation} examines whether fabricated or perturbed clues can steer the target LLM toward different generations. Traditional confession-oriented methods, such as the Reid technique, often involve false-evidence tactics \cite{inbau2013criminal}, which have been highly controversial in human interrogation. We explore whether a similar idea has an effect when interrogating LLMs. To reduce the risk of inducing hallucinations \cite{huang2025survey}, we introduce only minimal perturbations when constructing the prompts, rather than fabricating entirely new contexts.
Formally, the adversary constructs a set of prompts $\{p_i\}_{i=1}^{|\mathcal{C}^e_e|}$.
The adversary generates a corresponding set of new entities (unrelated to $e$) $\widetilde{\mathcal{C}}^e_e$, and constructs each prompt $p_i$ by replacing $e_i$ with $\widetilde{e}_i \in \widetilde{\mathcal{C}}^e_e$ (where $\widetilde{e}_i \neq \widetilde{e}_j$ for $i \neq j$).

In practice, the theoretical prompt space size $N_\pi$ varies significantly across different interrogation strategies, which determines our query count $n_\pi$ for each strategy.
For the first three strategies, exhaustively traversing their theoretical search spaces is computationally infeasible due to continuous domains or combinatorial explosion:
\begin{itemize}
    \item \textbf{Multirole:} since generation temperatures $\tau$ are drawn from a continuous range, the theoretical prompt space is infinite.
    \item \textbf{Repeated:} the number of unique permutations of related entities scales factorial as $|\mathcal{C}^e_e|!$.
    \item \textbf{Multiangle:} the power set of related entities leads to $2^{|\mathcal{C}^e_e|}$ possible subsets.
\end{itemize}
When $|\mathcal{C}^e_e|$ is large, traversing these spaces is prohibitive. Thus, for the three strategies, we restrict the practical query count to a fixed number $n_\pi = n$ prompts (where $n$ is a hyperparameter).
Conversely, the remaining two strategies can be exhaustively traversed within linear bounds, yielding exactly $|\mathcal{C}^e_e| + 1$ prompts for incremental appending ($\pi=gradual$) and $|\mathcal{C}^e_e|$ prompts for entity perturbation ($\pi=misleading$).

\subsection{Membership Inference based on Related Information Leakage}

Based on the intuition of related-information leakage, we perform membership inference using the sequence of $n_\pi$ textual responses obtained by querying the target LLM $\mathcal{M}$ with $n_\pi$ constructed prompts.
A naive idea would be to directly quantify the amount of entity-related information elicited from $\mathcal{M}$. However, defining and compiling a complete, closed-form set of ground-truth reference texts for any given target entity is inherently impractical. To overcome this hurdle, we instead pair the generated responses and calculate their semantic similarities, which serves as a tractable, indirect proxy.
Notably, while the multirole interrogation strategy yields $n$ texts across $n$ temperatures, the greedy-decoded response $r_{\text{base}}$ (representing an additional witness's response generated under the same prompt) is also integrated into the pairing process. 
Thus, we define the text sequence to be paired $\mathcal{R^\pi}$ as $[r^\pi_1, \dots, r^\pi_{n_\pi}]$ where $\pi \neq multirole$ and $[r_{\text{base}}, r^\pi_1, \dots, r^\pi_{n_\pi}]$ where $\pi = multirole$.

To align with the specific elicitation logic of different strategies, we introduce two distinct text pairing schemes (i.e., Pairwise Consistency and Sequential Transition) to calculate the semantic similarity of the texts and construct the similarity vector $\mathbf{v}_\pi$ for $e$. 
\textbf{Pairwise Consistency.} Inspired by previous membership inference \cite{hu2025membership}, we evaluate the overall semantic consistency across all responses for four strategies except gradual interrogation, yielding a vector of dimension $\frac{|\mathcal{R^\pi}|(|\mathcal{R^\pi}| - 1)}{2}$. This scheme operates on the premise that if $e$ is a member, $\mathcal{M}$'s responses across different prompts will consistently gravitate toward the same underlying factual core, yielding high global similarity. For non-members, the lack of a shared knowledge base will instead result in low pairwise similarity.
\textbf{Sequential Transition.} For gradual interrogation, we track the step-by-step marginal effect of gradually increasing information. Hence, we only compare each response $r^\pi_i$ with its immediate predecessor $r^\pi_{i-1}$, yielding a vector of dimension $|\mathcal{R^\pi}| - 1 = |\mathcal{C}^e_e|$. This sequential pairing allows us to capture how the similarity of the model's responses shifts as clue density increases. 
Thus, we define $\mathbf{v}_\pi$ as:
\begin{equation}
\mathbf{v}_\pi = 
\begin{cases}
\left[ \text{sim}(r^\pi_i, r^\pi_j) \right]_{1 \le i < j \le |\mathcal{R^\pi}|}, & \text{for } \pi \neq gradual \\
\left[ \text{sim}(r^\pi_{i-1}, r^\pi_i) \right]_{i=2}^{|\mathcal{R^\pi}|}, & \text{for } \pi = gradual
\end{cases}
\end{equation}
where $\text{sim}(\cdot, \cdot)$ is a semantic similarity function.

Finally, the similarity vector $\mathbf{v}_\pi$ is used as input to a binary classifier $\mathcal{BC}$, which returns a prediction for the membership status of the target entity $e$. 
Specifically, we consider three widely-adopted classification models for $\mathcal{BC}$ representing distinct machine learning paradigms, namely Random Forest (RF), Logistic Regression (LR), and Multi-Layer Perceptron (MLP), along with a simple thresholding classifier based on the Mean Value (MV) of similarity scores.
For each strategy, we train $\mathcal{BC}$ on a training set consisting of known member and non-member entities, and subsequently evaluate its binary classification performance on a disjoint set of unseen testing entities.

\section{Evaluation}

In this section, we first outline our experimental setup and evaluate the effectiveness of our entity-level membership inference. We then assess its applicability across target LLMs of varying scales and architectures, followed by a parameter analysis. Finally, we present a case study to demonstrate its practical utility.

\subsection{Experimental Setup}

The experiments are conducted on a desktop computer with Ubuntu 22.04LTS operating system, 3.0 GHz Intel® Xeon® Gold 6248R CPU, 251GB memory, and two NVIDIA RTX A6000 GPUs.

\subsubsection{Datasets}

To evaluate the effectiveness of our proposed methods in inferring entity-level membership, we apply the Named Entity Recognition (NER) technique \cite{lample2016neural, li2020survey} to identify entities from texts and target two representative categories of entities: \emph{PERSON} and \emph{ORG}. The selection of these categories is rooted in contemporary privacy and copyright concerns. Specifically, the exposure of \emph{PERSON} related information may not only compromise individual privacy (such as PII leakage), but also infringe upon intellectual property (such as proprietary roles within copyrighted textual works). Similarly, the exposure of \emph{ORG} related information may jeopardize corporate confidentiality (e.g., sensitive user data) and violate institutional copyrighted assets. For each entity type, we create a balanced dataset containing a fixed number of member entities alongside an equal number of constructed non-member entities.

\textbf{Member Entity Recognition.}
Identifying member entities is relatively straightforward, as many LLMs are pretrained on widely used corpora. Inspired by document inference \cite{meeus2024did}, we focus on a representative benchmark, 
RedPajama-Data-1T \cite{weber2024redpajama}. It comprises seven distinct data sources, such as Wikipedia \cite{Wikipedia}, C4 \cite{raffel2020exploring}, GitHub \cite{GitHub}, Books (Project Gutenberg \cite{ProjectGutenberg} and Books3 \cite{reisner2023these, knibbs2023battle}). RedPajama-Data-1T is released as an open-source dataset, with the sole exception of the Books3 subset, which has been removed from public distribution due to copyright infringement complaints. 

Specifically, we first use the $en\_core\_web\_trf$ model of the spaCy library \cite{spaCy} (which is a widely used NLP tool \cite{prasad2023diving, shoaib2025principled, yan2025skillpov, yan2025no}) to extract entities from Wikipedia in RedPajama-Data-1T (a total of 6,630,651 Wikipedia pages in English). For each page, if the first extracted entity $e_m$ within the body text belongs to \emph{PERSON} or \emph{ORG} and shares a lexical overlap with the page's title (to accommodate partial name matches in either the title or the text), we consider the page to be primarily centered around $e_m$. Thereby, we designate $e_m$ as a candidate member entity and all other entities co-occurring within the same body text as candidate related entities (serving as contextual cues).
For each entity category, we randomly sample 1,000 entities to constitute the set of member entities within our dataset. To guarantee data quality, each candidate target entity is subjected to a secondary verification process via Gemini 3 Flash \cite{team2023gemini}. Any entity deemed mismatched with its assigned category (e.g., an entity under the \emph{PERSON} category classified as non-\emph{PERSON} by Gemini 3 Flash) is filtered out. Furthermore, we select the first 20 unique entities that appear chronologically in the text as the related entity set $\mathcal{C}^e_{e_m}$ for \emph{PERSON} and the first 25 for \emph{ORG} (see Section \ref{num_related_entities} for details on the selection of the number of $\mathcal{C}^e_{e_m}$).

\textbf{Non-member Entity Construction.}
Constructing non-member entities is more challenging. To minimize distribution shift and near-member effects, we construct non-member target entities from member entities, ensuring that the resulting entities remain close to the members while being unlikely to appear in the training data. 

Specifically, to construct a candidate non-member target entity $e_{nm}$, we randomly sample two pages and extract a single target-category entity from each, denoted as $e^1_{m}$ and $e^2_{m}$. We partition both entities at their respective midpoints based on space-delimited word counts, utilizing floor division ($\lfloor L/2 \rfloor$, where $L$ represents the word length) to determine the boundary. We then concatenate the first half of $e^1_{m}$ and the second half of $e^2_{m}$ to yield the new combined entity $e_{nm}$.
Next, we collect entities from different Wikipedia pages as clues $\mathcal{C}^e_{e_{nm}}$ to $e_{nm}$. Pages are randomly selected one at a time, with only a single entity being extracted from each page, until the total counts of 20 for \emph{PERSON} and 25 for \emph{ORG} are reached.

Consistent with the processing of member entities, we employ Gemini 3 Flash to perform a secondary verification on each $e_{nm}$. Additionally, to mitigate potential false positives, we search through the publicly available sub-datasets within RedPajama-Data-1T, filtering out any candidate cases where $e_{nm}$ and its associated clues $\mathcal{C}^e_{e_{nm}}$ co-occur in the same text. In practice, we observe that this filter is highly stringent: searching with merely $e_{nm}$ alongside five entities in $\mathcal{C}^e_{e_{nm}}$ already yields zero co-occurrences. We note one exception regarding the Books3 subset: due to its official withdrawal from public distribution, we are unable to perform this search-based filtering on its raw texts. Out of strict respect for copyright compliance and to uphold reproducible research standards, we refrain from utilizing unauthorized external mirrors. Nevertheless, given that zero co-occurrences are consistently achieved across the other massive subsets, we are confident that the probability of $e_{nm}$ and $\mathcal{C}^e_{e_{nm}}$ accidentally co-occurring in Books3 is statistically negligible and does not compromise our evaluation framework.

\textbf{Training and Evaluation Set Partitioning.}
For each entity category, we randomly partition the dataset of 1,000 members and 1,000 non-members $D^1_e$ into training and evaluation sets in a 4:1 ratio. The training set is used both to train the binary classifier $\mathcal{BC}$ and to perform hyperparameter tuning. Detailed hyperparameter configurations for $\mathcal{BC}$ are provided in Appendix \ref{appendix_classifiers}. Specifically, we perform a $K$-fold cross-validation ($K=5$) for each hyperparameter configuration, evaluating each candidate set by averaging its scores (we use AUC) in the $K$ validation folds. The configuration that yields the highest average score is selected as the optimal hyperparameter set. Using these optimal hyperparameters, we then retrain the final model on the entire training dataset. In contrast, the evaluation set is reserved strictly for a single, final evaluation to rigorously assess the model's generalization capability on entirely unseen data. 
Additionally, we randomly sample 10\% each of the member and non-member entities from the training set to form the dataset $D^2_e \subset D^1_e$ for our parameter analysis.

\subsubsection{Target Models}

For our evaluations, we use two distinct families of open-source LLMs pretrained only on RedPajama-Data-1T \cite{weber2024redpajama}: OpenLLaMA \cite{geng2023openllama} and the base model of RedPajama-INCITE \cite{weber2024redpajama}. OpenLLaMA is a fully open-source reproduction of Meta's LLaMA \cite{touvron2023llama}. It is publicly available in three parameter scales: 3B, 7B, and 13 B. The base model of RedPajama-INCITE adopts the Pythia model architecture \cite{biderman2023pythia}. It has two scaled variants consisting of 3B and 7B parameters.

\subsubsection{Other Parameter and Implementation Details}
The following are the four parameter details and two implementation details of our proposed methods.

\textbf{Parameter Details.}
(a) For the multirole interrogation strategy, we employ the designated temperature set \{0.01, 0.05, 0.1, 0.2, 0.3, 0.4, 0.5, 0.6, 0.7, 0.8, 0.9, 1.0, 1.2, 1.4, 1.6, 1.8\} proposed by Hu et al. \cite{hu2025membership}. In their work, the authors evaluate the outputs of Vision-Language Models on target samples across various temperatures.
(b) The number of related entities for \emph{PERSON} and \emph{ORG} is set to 20 and 25, respectively (see Section~\ref{num_related_entities} for details).
(c) The maximum token generation length for the target LLM is capped at 256 (see Section~\ref{max_new_tokens} for details).
(d) In Section~\ref{prompt_construction}, we analyze that the theoretical prompt spaces of the three interrogation strategies (i.e., multirole, repeated, and multiangle) are either continuous or subject to combinatorial explosion. In our experiments, we restrict the actual number of prompts to 14 (see Section~\ref{num_prompts} for details). 

\textbf{Implementation Details.}
(a) For the misleading strategy, we generate a new entity set $\widetilde{\mathcal{C}}^e_e$ for each target entity $e$. For a member entity $e_m$, we randomly select a page text (excluding its source text $t_m$) and extract a random entity $e'$ from it in each iteration. Under the condition that $e'$ does not appear in $t_m$, we append it to $\widetilde{\mathcal{C}}^e_{e_m}$. We enforce the constraint $|\widetilde{\mathcal{C}}^e_{e_m}|=|\mathcal{C}^e_{e_m}|$, thereby ensuring that the replacement entities utilized to construct misleading prompts are unique. Similarly, for a non-member entity $e_{nm}$, we randomly select a page text, extract a random entity, and add it to $\widetilde{\mathcal{C}}^e_{e_{nm}}$ provided that it is absent from both the original clue set $\mathcal{C}^e_{e_{nm}}$ and the growing set $\widetilde{\mathcal{C}}^e_{e_{nm}}$ itself.
(b) We generate embeddings for paired texts and compute their semantic similarity using the sentence-transformers/all-MiniLM-L6-v2 model \cite{reimers2019sentence}, which is widely adopted for measuring text semantic similarity~\cite{wen2024membership,he2025towards,lu2026context}.

\subsubsection{Baseline Methods}

To the best of our knowledge, we are the first to propose entity-level membership inference, leaving no direct baselines for comparison. To establish benchmarks, we leverage state-of-the-art label-only sample-level membership inference methods (i.e., PETAL \cite{he2025towards} and the label-only variant of ICP-MIA \cite{lu2026context}, see Section \ref{background_mia} for details). To accommodate their sample-level input requirement, we convert each target entity and its related entities into a 256-word text sample.

Specifically, for member entities, we retrieve the Wikipedia source text housing the target and its related entities, retaining only the first 256 words. Since the target is the first identified entity and its related entities are captured chronologically thereafter, this 256-word window is guaranteed to encompass the entire entity set. For non-member entities, inspired by ICP-MIA \cite{lu2026context}, we employ Qwen-2.5-72B-Instruct to generate synthetic text containing the target and its related entities (prompts are detailed in Appendix~\ref{appendix_prompt}). We set \textit{max\_new\_tokens} to 512 and subsequently truncate the generated output to the first 256 words to maintain length consistency.
Notably, because member samples originate from genuine human-written Wikipedia entries while non-member samples are synthesized by Qwen, the sample-level methods might exploit subtle stylistic discrepancies as a classification shortcut, thereby granting the baselines an additional inference advantage. See Appendix \ref{appendix_baselines} for detailed configuration information of baselines.

\subsubsection{Metrics}

We adopt three common metrics to evaluate MIAs.
\textbf{Balanced Accuracy} measures the attack’s performance on a balanced set of members and non-members \cite{shokri2017membership, choquette2021label}. It is defined as the average of the true positive rate (TPR) and true negative rate (TNR).
The area under the receiver operating characteristic curve (\textbf{AUC}) is a threshold-independent metric widely used for membership inference \cite{meeus2024did, he2025towards, hu2025membership} that summarizes the attack’s ability to distinguish members from non-members across all decision thresholds.
Following prior work \cite{carlini2022membership}, we also report the TPR at a low false positive rate (\textbf{TPR@low FPR}). It characterizes the attack success under a stringent false-alarm constraint and thus reflects high-confidence privacy leakage of the LLM's training data.

\subsection{Main Results}

\subsubsection{Comparison to Baselines}

\begin{table*}[!t]
\renewcommand{\arraystretch}{1.3}
\caption{Results of various inference targets and label-only inference methods on entity category labels \emph{PERSON} and \emph{ORG}. The target LLM is OpenLLaMA-7B. (MV: Mean Value, RF: Random Forest, LR: Logistic Regression)}
\label{table_comparison}
\centering
\setlength{\tabcolsep}{5pt}  
\begin{tabular}{ccccccccccc}
\hline
\multirow{3}{*}{Inference Target} & \multicolumn{2}{c}{\multirow{3}{*}{Inference Method}} & \multicolumn{4}{c}{\emph{PERSON}} & \multicolumn{4}{c}{\emph{ORG}}\\
\cline{4-11}
 & & & Balanced & \multirow{2}{*}{AUC} & \multicolumn{2}{c}{TPR@FPR} & Balanced & \multirow{2}{*}{AUC} & \multicolumn{2}{c}{TPR@FPR}\\ 
\cline{6-7} \cline{10-11}
 & & & Acc & & FPR=1\% & FPR=10\% & Acc & & FPR=1\%& FPR=10\%\\
\hline
\multirow{3}{*}{Sample-level} & \multicolumn{2}{c}{PETAL\cite{he2025towards}} & 0.755 & 0.800 & 0.380 & 0.585 & 0.673 & 0.685 & 0.190 & 0.390\\
\cline{2-11}
 & \multicolumn{2}{c}{ICP-MIA-SP (Masking-based) \cite{lu2026context}} & 0.603 & 0.598 & 0.020 & 0.190 & 0.558 & 0.548 & 0.005 & 0.050\\
 & \multicolumn{2}{c}{ICP-MIA-SP (Generation-based) \cite{lu2026context}} & 0.675 & 0.715 & 0.235 & 0.395 & 0.650 & 0.708 & 0.085 &  0.370\\
\hline
\multirow{20}{*}{Entity-level} & \multirow{4}{*}{multirole interrogation} & MV & 0.913 & 0.964 & 0.695 & 0.910 & 0.835 & 0.886 & 0.160 & 0.720\\
 &  & RF & 0.915 & 0.962 & 0.590 & 0.925 & 0.823 & 0.885 & 0.260 & 0.730\\
 &  & LR & 0.915 & 0.965 & 0.685 & 0.915 & \textbf{0.838} & 0.889 & 0.200 & \textbf{0.770}\\
 &  & MLP & 0.915 & 0.964 & 0.675 & 0.915 & 0.825 & 0.889 & \textbf{0.265} & 0.740\\
\cline{2-11}
 & \multirow{4}{*}{repeated interrogation} & MV & 0.920 & 0.967 & 0.305 & 0.920 & 0.805 & 0.861 & 0.025 & 0.610\\
 &  & RF & 0.928 & \textbf{0.970} & 0.265 & 0.925 & 0.828 & 0.877 & 0.050 & 0.560\\
 &  & LR & 0.923 & 0.967 & 0.290 & 0.915 & 0.808 & 0.861 & 0.025 & 0.605\\
 &  & MLP & \textbf{0.930} & 0.965 & 0.250 & \textbf{0.940} & 0.810 & 0.859 & 0.015 & 0.565\\
\cline{2-11}
 & \multirow{4}{*}{multiangle interrogation} & MV & 0.905 & 0.968 & 0.695 & 0.885 & 0.825 & 0.883 & 0.060 & 0.585\\
 &  & RF & 0.913 & 0.968 & \textbf{0.725} & 0.910 & 0.830 & \textbf{0.892} & 0.100 & 0.635\\
 &  & LR & 0.910 & 0.969 & 0.705 & 0.905 & 0.828 & 0.886 & 0.085 & 0.615\\
 &  & MLP & 0.913 & 0.967 & 0.700 & 0.920 & 0.820 & 0.886 & 0.105 & 0.625\\
\cline{2-11}
 & \multirow{4}{*}{gradual interrogation} & MV & 0.830 & 0.872 & 0.215 & 0.535 & 0.750 & 0.807 & 0.065 & 0.490\\
 &  & RF & 0.823 & 0.885 & 0.095 & 0.660 & 0.775 & 0.826 & 0.035 & 0.455\\
 &  & LR & 0.815 & 0.879 & 0.170 & 0.560 & 0.740 & 0.807 & 0.060 & 0.460\\
 &  & MLP & 0.823 & 0.881 & 0.240 & 0.540 & 0.750 & 0.814 & 0.035 & 0.465\\
\cline{2-11}
 & \multirow{4}{*}{misleading interrogation} & MV & 0.847 & 0.910 & 0.050 & 0.725 & 0.723 & 0.760 & 0.030 & 0.275\\
 &  & RF & 0.885 & 0.943 & 0.105 & 0.855 & 0.733 & 0.801 & 0.160 & 0.425\\
 &  & LR & 0.878 & 0.934 & 0.150 & 0.810 & 0.713 & 0.766 & 0.050 & 0.320\\
 &  & MLP & 0.878 & 0.931 & 0.060 & 0.830 & 0.725 & 0.755 & 0.000 & 0.235\\
\cline{2-11}
\hline
\end{tabular}
\end{table*}

This experiment details the performance of our proposed entity-level membership inference on OpenLLaMA-7B across two representative entity categories (\emph{PERSON} and \emph{ORG}), comparing it with three state-of-the-art label-only sample-level membership inference methods: PETAL \cite{he2025towards}, Masking-based ICP-MIA-SP \cite{lu2026context}, and Generation-based ICP-MIA-SP \cite{lu2026context}. To facilitate evaluation, these sample-level baselines are performed on text samples constructed to encapsulate the target entities and their related clues. As shown in Table \ref{table_comparison}, our entity-level inference significantly outperforms the label-only sample-level baselines.

On the \emph{PERSON} category, our framework utilizing all five strategies outperforms the strongest baseline (PETAL) by 6.0\%--17.5\% in Balanced Accuracy and by 0.072--0.17 in AUC. Specifically, the repeated strategy with MLP achieves the highest Balanced Accuracy of 93.0\%, while the same strategy with RF yields a peak AUC of 0.97. In TPR@1\%FPR, our multirole and multiangle strategies significantly surpass PETAL (which scores 38.0\%) by 21\%--34.5\%, whereas the remaining three strategies show lower utility. In TPR@10\%FPR, four strategies outperform PETAL by 14\%--35.5\%, while the gradual strategy performs comparably to the baseline. Among the baselines, PETAL is the strongest, 
followed by generation-based ICP-MIA-SP. The masking-based variant performs worst due to random masking's severe semantic destruction.

In addition, the performance of all inference methods is lower on \emph{ORG} than on \emph{PERSON}. We attribute this cross-category performance gap to the differing data profiles of these entities in LLMs' pre-training corpora. Biographical information surrounding specific individuals is typically highly localized, unique, and concentrated within cohesive narrative structures (e.g., biography books). In contrast, organizational data frequently appears in scattered, diverse, and generic contexts (e.g., financial reports and news aggregators). This dispersion leads to weaker factual association networks in the model, making memory retrieval more difficult. Despite this, our entity-level inference maintains its superiority on \emph{ORG}. All five proposed strategies consistently outperform baselines in both Balanced Accuracy and AUC. In TPR@1\%FPR, the multirole strategy excels over the strongest baseline across RF, LR, and MLP. In TPR@10\%FPR, four of our proposed strategies exceed the best baseline.

\subsubsection{Effectiveness Analysis of Different Interrogation Strategies}

\begin{figure*}[!t]
\centering
\subfloat[PCA]{\includegraphics[width=0.5\textwidth]{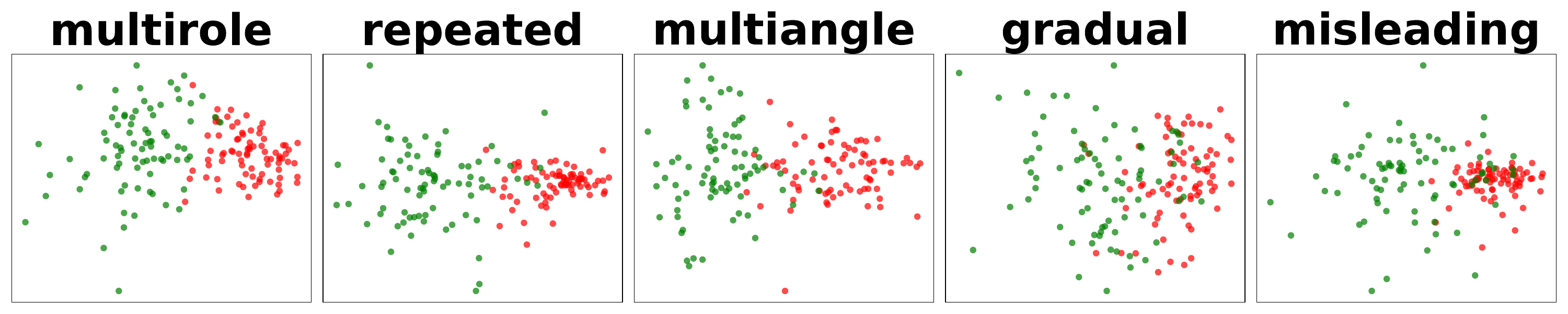}%
\label{fig_pca}}
\subfloat[t-SNE]{\includegraphics[width=0.5\textwidth]{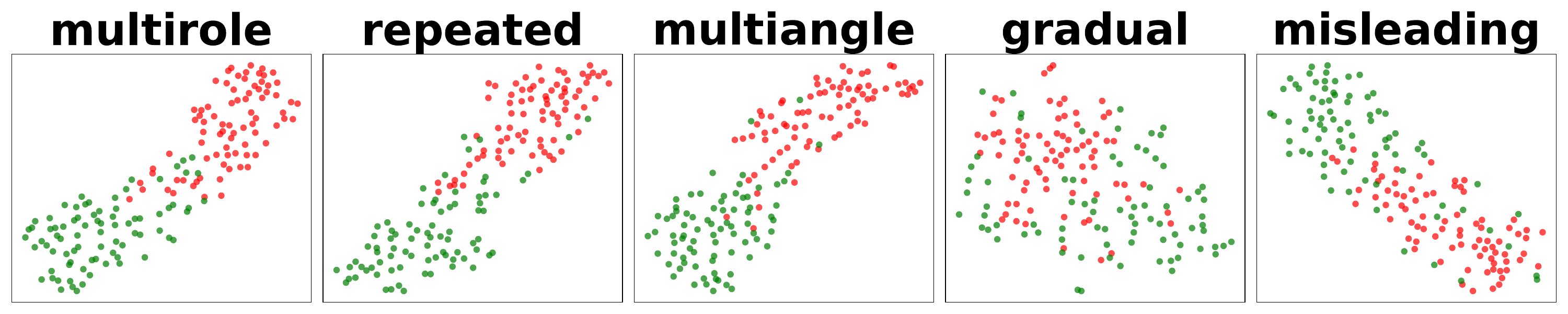}%
\label{fig_tsne}}
\caption{PCA and t-SNE visualizations of semantic similarity vectors for generated texts of member vs. non-member entities on \emph{PERSON} across five strategies. Red data points: members. Green data points: non-members.}
\label{fig_scatter}
\end{figure*}

We delve into the performance discrepancies among our five proposed interrogation strategies. As shown in Table \ref{table_comparison}, on the \emph{PERSON} category, multirole, repeated, and multiangle strategies exhibit outstanding and comparable performance. Specifically, the repeated strategy dominates in Balanced Accuracy, AUC, and TPR@10\%FPR. However, its TPR@1\%FPR is over 28.5\% lower than the other two, while the multiangle strategy achieves the peak performance of 72.5\%. We hypothesize that while the multiangle strategy alters the semantic substance of the queries, the repeated strategy merely permutes the order of identical lexical components. LLMs are sensitive to identical context permutations (known as permutation instability \cite{lu2022fantastically}), which occasionally trigger erratic attention shifts and produce highly divergent response outliers for a minor subset of member samples. These outliers heavily penalize TPR@1\%FPR but have a negligible impact on average metrics like AUC. The second tier is occupied by the misleading strategy, followed by the gradual strategy at the bottom. In contrast, on the \emph{ORG} category, the same top three strategies maintain their dominance, while the second and third tiers swap. This shift is rooted in the highly diffuse and weakly bounded nature of \emph{ORG} training data, making the model susceptible to the adversarial suggestibility of the misleading strategy.

To intuitively investigate these performance discrepancies, we project the semantic similarity vectors of $D^2_e$ for \emph{PERSON} using PCA \cite{mackiewicz1993principal} (Fig.~\ref{fig_pca}) and t-SNE \cite{van2008visualizing} (Fig.~\ref{fig_tsne}), where red and green points denote member and non-member entities, respectively.
Across both visualizations, members and non-members exhibit distinct regional clustering, particularly under the top-performing multirole, repeated, and multiangle strategies. While the remaining two strategies show weaker partitioning, the misleading strategy exhibits a visibly cleaner separation than the gradual strategy. This arises because the gradual strategy (utilizing Sequential Transition) captures the trajectory of shifting similarities rather than the absolute similarity levels evaluated under Pairwise Consistency. 
Overall, this geometric separation aligns with the performance ranking in Table~\ref{table_comparison}, confirming that our generated similarity vectors are highly discriminative, which simplifies the decision boundary.
Notably, non-member points display significantly higher spatial dispersion than member points (Fig.~\ref{fig_pca}). This occurs because querying a memorized entity consistently activates a cohesive internal memory, producing stable and coherent outputs that cluster tightly in the vector space. Conversely, for an unmemorized entity, the model lacks the corresponding internal knowledge representation and fails to produce consistent, entity-specific information, leading to scattered vector coordinates.

\subsection{Applicability to Different Target LLMs}

\begin{table*}[!t]
\renewcommand{\arraystretch}{1.3}
\caption{Results on Different Target LLMs. The entity category label is \emph{PERSON}.}
\label{table_applicability}
\centering
\setlength{\tabcolsep}{5pt}  
\begin{tabular}{cccccccccccccc}
\hline
Interrogation & Target & \multicolumn{4}{c}{Balanced Acc} & \multicolumn{4}{c}{AUC} & \multicolumn{4}{c}{TPR@1\%FPR}\\ 
\cline{3-14}
Strategy & LLM & MV & RF & LR & MLP & MV & RF & LR & MLP & MV & RF & LR & MLP\\
\hline
\multirow{5}{*}{multirole} & OpenLLaMA-3B & 0.888 & 0.885 & 0.885 & 0.885 & 0.955 & 0.956 & 0.958 & 0.954 & 0.630 & 0.670 & 0.665 & 0.590\\
 & OpenLLaMA-7B & 0.913 & 0.915 & 0.915 & 0.915 & 0.964 & 0.962 & 0.965 & 0.964 & 0.695 & 0.590 & 0.685 & 0.675\\
 & OpenLLaMA-13B & 0.895 & 0.900 & 0.898 & 0.898 & 0.957 & 0.956 & 0.960 & 0.961 & 0.535 & 0.635 & 0.630 & 0.650\\
 & RedPajama-INCITE-Base-3B-v1 & 0.913 & 0.913 & 0.920 & 0.905 & 0.964 & 0.959 & 0.963 & 0.955 & 0.460 & 0.475 & 0.470 & 0.485\\
 & RedPajama-INCITE-7B-Base & 0.903 & 0.905 & 0.903 & 0.893 & 0.963 & 0.963 & 0.961 & 0.949 & \textbf{0.730} & 0.565 & 0.640 & 0.560\\
\hline
\multirow{5}{*}{repeated} & OpenLLaMA-3B & 0.895 & 0.895 & 0.895 & 0.898 & 0.947 & 0.954 & 0.947 & 0.947 & 0.105 & 0.380 & 0.110 & 0.170\\
 & OpenLLaMA-7B & \textbf{0.920} & \textbf{0.928} & \textbf{0.923} & \textbf{0.930} & 0.967 & 0.970 & 0.967 & 0.965 & 0.305 & 0.265 & 0.290 & 0.250\\
 & OpenLLaMA-13B & 0.915 & 0.923 & 0.918 & 0.920 & \textbf{0.970} & \textbf{0.973} & \textbf{0.970} & \textbf{0.971} & 0.495 & 0.495 & 0.495 & 0.530\\
 & RedPajama-INCITE-Base-3B-v1 & 0.890 & 0.890 & 0.888 & 0.890 & 0.953 & 0.961 & 0.952 & 0.949 & 0.395 & 0.595 & 0.395 & 0.375\\
 & RedPajama-INCITE-7B-Base & 0.908 & 0.900 & 0.908 & 0.905 & 0.960 & 0.961 & 0.960 & 0.958 & 0.285 & 0.430 & 0.265 & 0.065\\
\hline
\multirow{5}{*}{multiangle} & OpenLLaMA-3B & 0.895 & 0.900 & 0.900 & 0.903 & 0.953 & 0.954 & 0.955 & 0.955 & 0.450 & 0.405 & 0.455 & 0.455\\
 & OpenLLaMA-7B & 0.905 & 0.913 & 0.910 & 0.913 & 0.968 & 0.968 & 0.969 & 0.967 & 0.695 & \textbf{0.725} & \textbf{0.705} & \textbf{0.700}\\
 & OpenLLaMA-13B & 0.898 & 0.903 & 0.905 & 0.893 & 0.968 & 0.970 & \textbf{0.970} & 0.966 & 0.720 & 0.700 & 0.690 & 0.585\\
 & RedPajama-INCITE-Base-3B-v1 & 0.895 & 0.890 & 0.895 & 0.888 & 0.950 & 0.947 & 0.952 & 0.949 & 0.450 & 0.405 & 0.460 & 0.470\\
 & RedPajama-INCITE-7B-Base & 0.913 & 0.913 & 0.915 & 0.918 & 0.968 & \textbf{0.973} & 0.969 & 0.968 & 0.575 & 0.695 & 0.585 & 0.655\\
\hline
\multirow{5}{*}{gradual} & OpenLLaMA-3B & 0.790 & 0.790 & 0.785 & 0.788 & 0.851 & 0.856 & 0.855 & 0.858 & 0.105 & 0.220 & 0.080 & 0.155\\
 & OpenLLaMA-7B & 0.830 & 0.823 & 0.815 & 0.823 & 0.872 & 0.885 & 0.879 & 0.881 & 0.215 & 0.095 & 0.170 & 0.240\\
 & OpenLLaMA-13B & 0.780 & 0.813 & 0.775 & 0.785 & 0.864 & 0.879 & 0.864 & 0.866 & 0.325 & 0.295 & 0.375 & 0.355\\
 & RedPajama-INCITE-Base-3B-v1 & 0.803 & 0.828 & 0.820 & 0.820 & 0.864 & 0.885 & 0.875 & 0.878 & 0.115 & 0.130 & 0.070 & 0.110\\
 & RedPajama-INCITE-7B-Base & 0.815 & 0.825 & 0.813 & 0.820 & 0.858 & 0.886 & 0.866 & 0.860 & 0.000 & 0.140 & 0.000 & 0.000\\
\hline
\multirow{5}{*}{misleading} & OpenLLaMA-3B & 0.800 & 0.845 & 0.845 & 0.830 & 0.876 & 0.916 & 0.906 & 0.895 & 0.085 & 0.285 & 0.060 & 0.120\\
 & OpenLLaMA-7B & 0.847 & 0.885 & 0.878 & 0.878 & 0.910 & 0.943 & 0.934 & 0.931 & 0.050 & 0.105 & 0.150 & 0.060\\
 & OpenLLaMA-13B & 0.823 & 0.855 & 0.853 & 0.845 & 0.885 & 0.929 & 0.910 & 0.897 & 0.235 & 0.225 & 0.195 & 0.145\\
 & RedPajama-INCITE-Base-3B-v1 & 0.833 & 0.878 & 0.863 & 0.860 & 0.888 & 0.926 & 0.914 & 0.908 & 0.130 & 0.265 & 0.195 & 0.180\\
 & RedPajama-INCITE-7B-Base & 0.855 & 0.880 & 0.873 & 0.885 & 0.922 & 0.939 & 0.933 & 0.933 & 0.225 & 0.415 & 0.290 & 0.340\\
\hline
\end{tabular}
\end{table*}

In this experiment, we compare target LLMs across varying parameter scales and neural architectures to assess the applicability of our proposed entity-level inference. The hyperparameters of our method are optimized on OpenLLaMA-7B and transferred to other target LLMs without additional tuning.

As shown in Table \ref{table_applicability}, within the same architecture, OpenLLaMA-7B delivers the best overall performance in most cases. A notable exception occurs under the repeated strategy, where OpenLLaMA-13B exceeds OpenLLaMA-7B in both AUC and TPR@1\%FPR. This exception arises because while hyperparameter transfer slightly penalizes performance, the superior learning and memorization capacity of the 13B model provides a countervailing gain. This capacity-driven benefit is further corroborated by the fact that OpenLLaMA-13B consistently outperforms OpenLLaMA-3B. Across different architectures, OpenLLaMA-7B and RedPajama-INCITE-7B-Base exhibit highly comparable and competitive performance.

Overall, our framework maintains stability in Balanced Accuracy and AUC when transferring across model scales and architectures, with maximum discrepancies of 5.0\% and 0.036, respectively. However, TPR@1\%FPR is more sensitive to these variations, exhibiting a maximum discrepancy of up to 32.0\%. This indicates that while overall discriminative power (e.g., AUC and Accuracy) generalizes robustly across diverse model scales and architectures, precise low-FPR decision boundaries would benefit from model-specific hyperparameter calibration.

\subsection{Parameter Analysis}

In this section, we investigate the impact of three key hyperparameters on entity-level inference performance: the number of related entities $|\mathcal{C}^e_e|$, the maximum generation token length \textit{max\_new\_tokens}, and the constructed prompt number $n_p$. We report the optimal $K$-fold cross-validation AUC scores for MV and RF classifiers.

\subsubsection{Number of Related Entities}
\label{num_related_entities}

\begin{figure*}[!t]
\centering
\subfloat[\emph{PERSON}: Mean Value]{\includegraphics[width=0.24\textwidth]{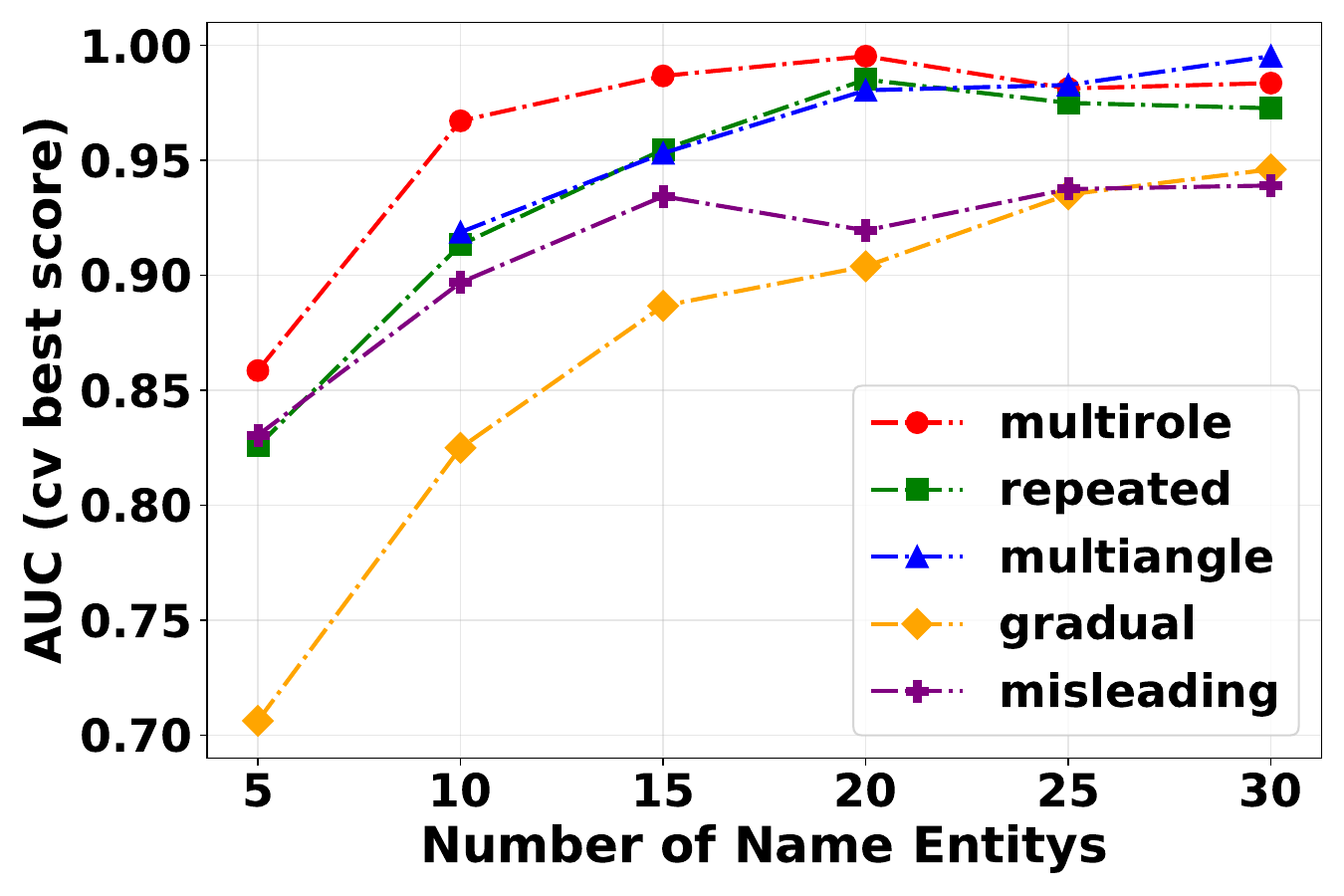}
\label{fig_ne_person_mv}}
\subfloat[\emph{PERSON}: Random Forest]{\includegraphics[width=0.24\textwidth]{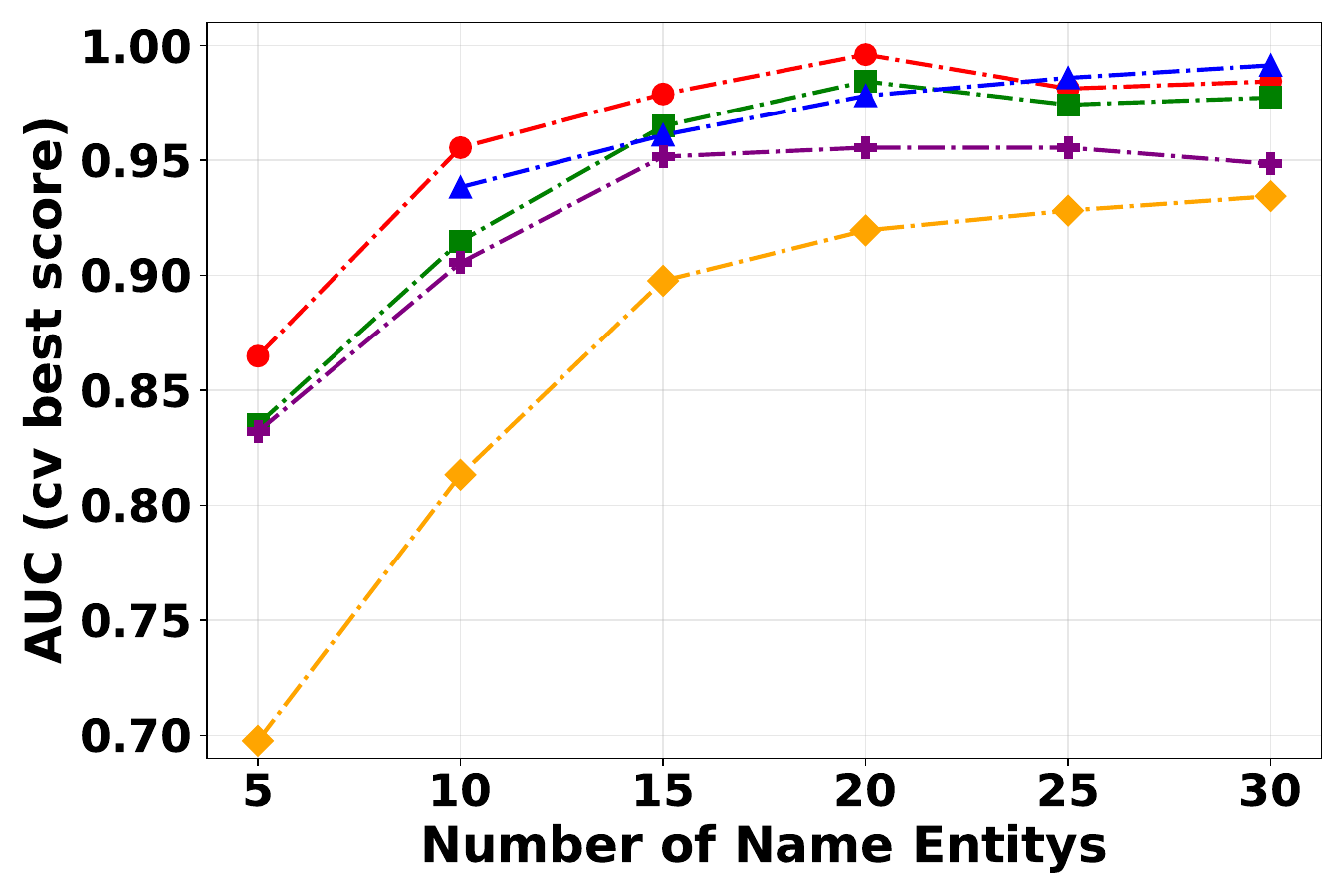}
\label{fig_ne_person_rf}}
\subfloat[\emph{ORG}: Mean Value]{\includegraphics[width=0.24\textwidth]{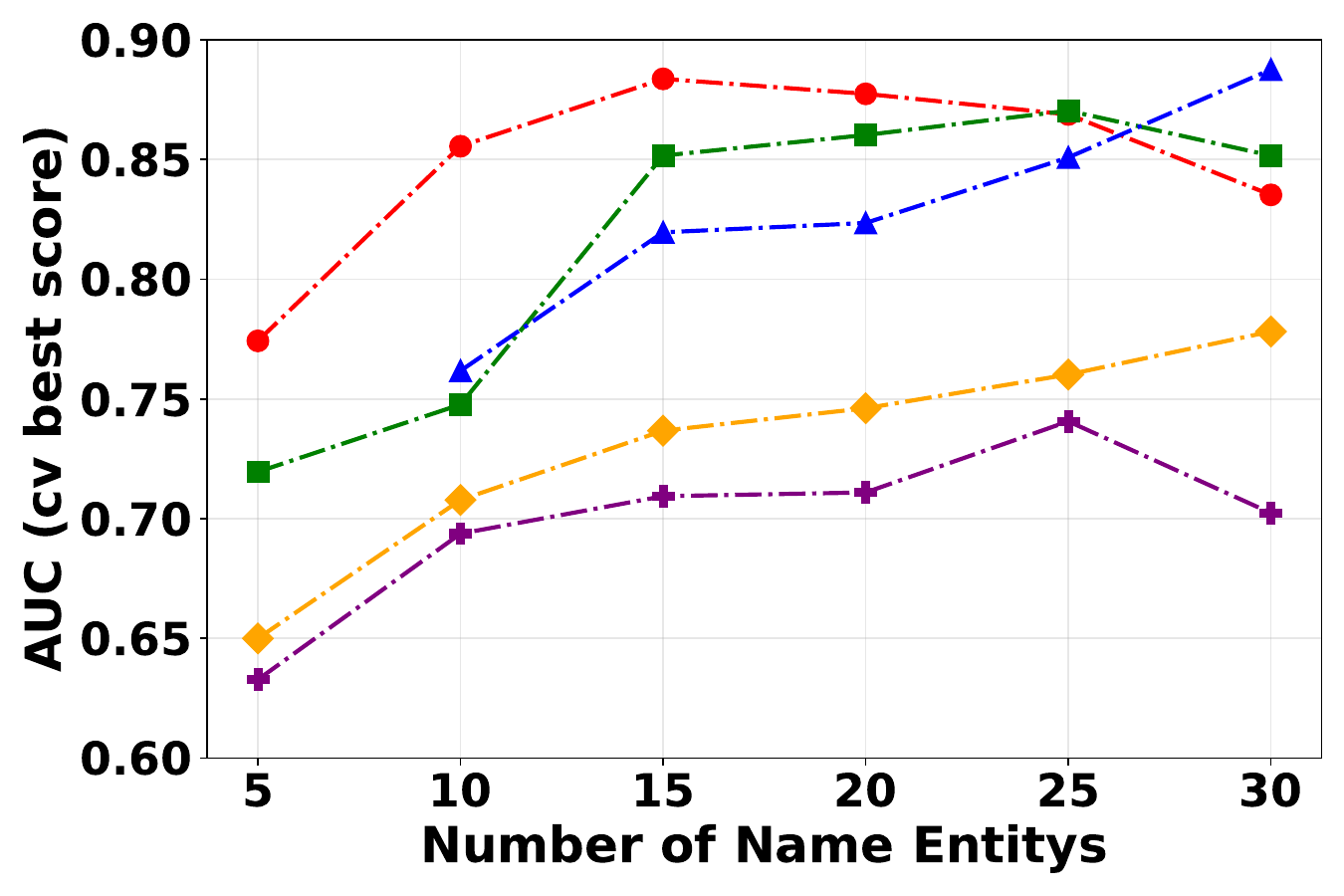}
\label{fig_ne_org_mv}}
\subfloat[\emph{ORG}: Random Forest]{\includegraphics[width=0.24\textwidth]{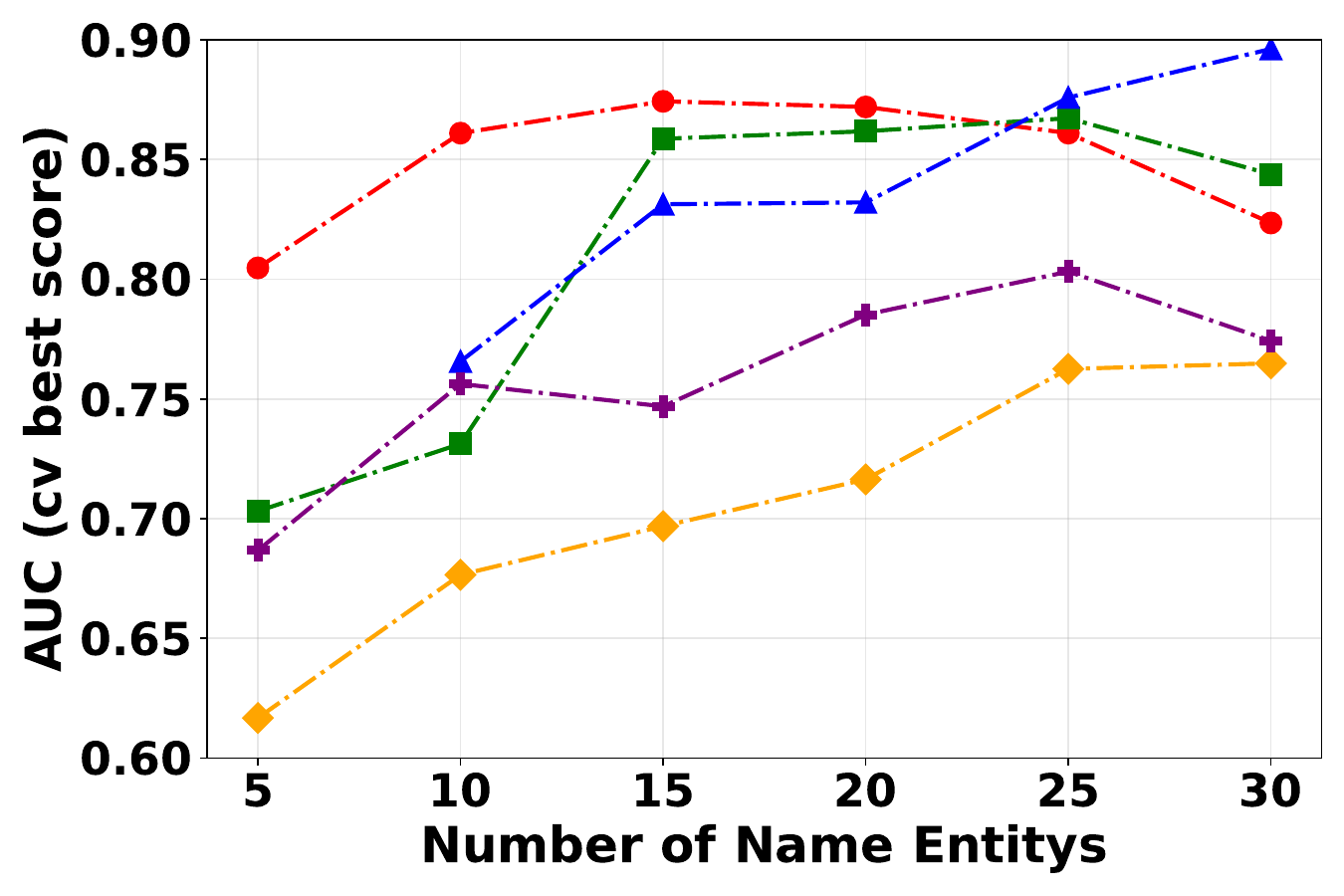}
\label{fig_ne_org_rf}}
\caption{Impact of different numbers of related entities on entity category labels \emph{PERSON} and \emph{ORG}.}
\label{fig_ne}
\end{figure*}

We evaluate $|\mathcal{C}^e_e|$ from 5 to 30 at intervals of 5 (Fig.~\ref{fig_ne}). Generally, AUC rises with $|\mathcal{C}^e_e|$ as more clues activate latent memory pathways. For \emph{PERSON} (Figs.~\ref{fig_ne_person_mv} and \ref{fig_ne_person_rf}), performance rises sharply before 20 (and we select $|\mathcal{C}^e_e|$=20). For \emph{ORG} (Figs.~\ref{fig_ne_org_mv} and \ref{fig_ne_org_rf}), most strategies peak or stabilize at $|\mathcal{C}^e_e|$=25. Note that the multiangle strategy lacks data at $|\mathcal{C}^e_e|$=5 because 5 clues can yield at most 10 unique combinations, failing to meet our default prompt count requirement of 14.

\subsubsection{Maximum Length of Generated Tokens}
\label{max_new_tokens}

\begin{figure*}[!t]
\centering
\subfloat[\emph{PERSON}: Mean Value]{\includegraphics[width=0.24\textwidth]{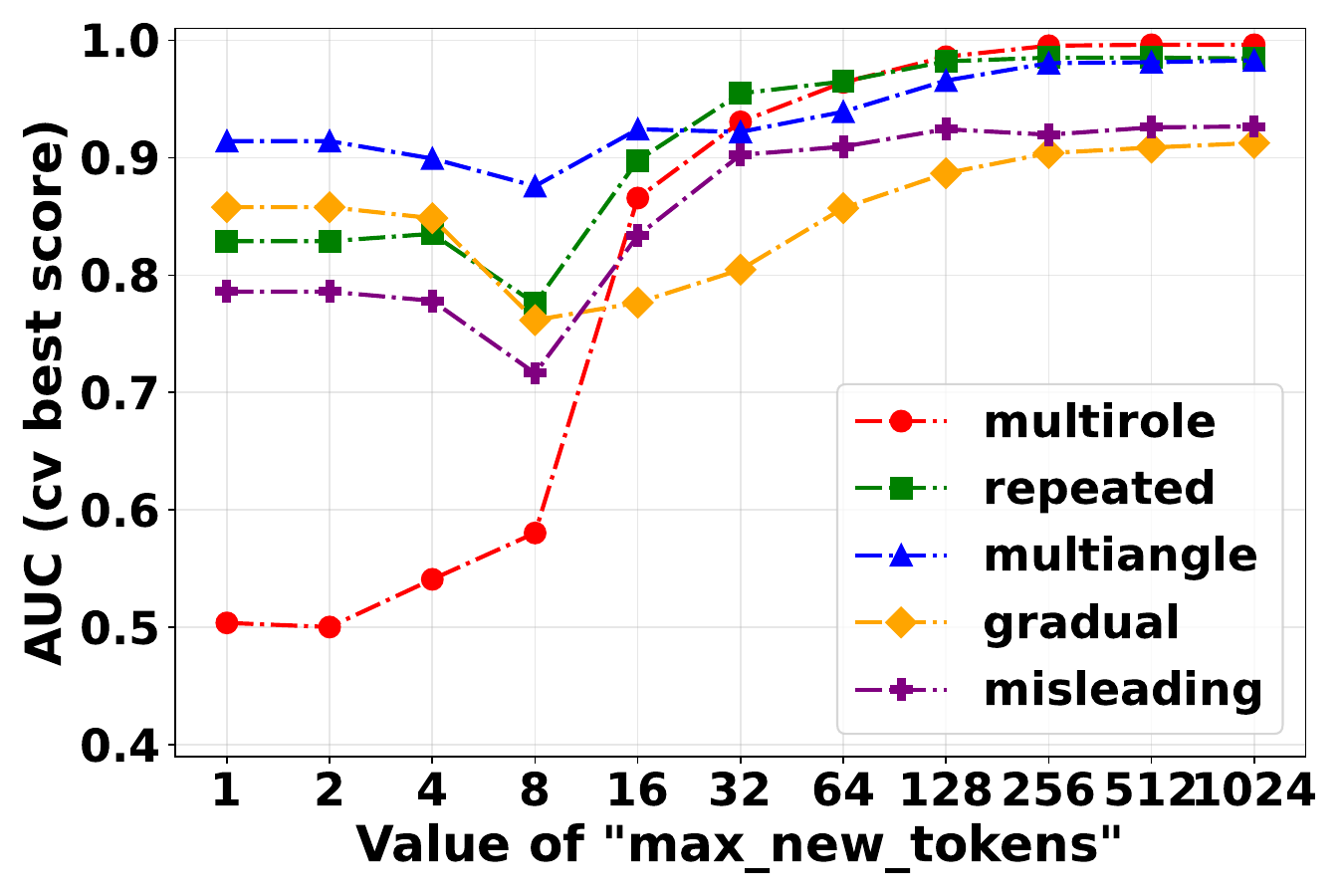}
\label{fig_mnt_person_mv}}
\subfloat[\emph{PERSON}: Random Forest]{\includegraphics[width=0.24\textwidth]{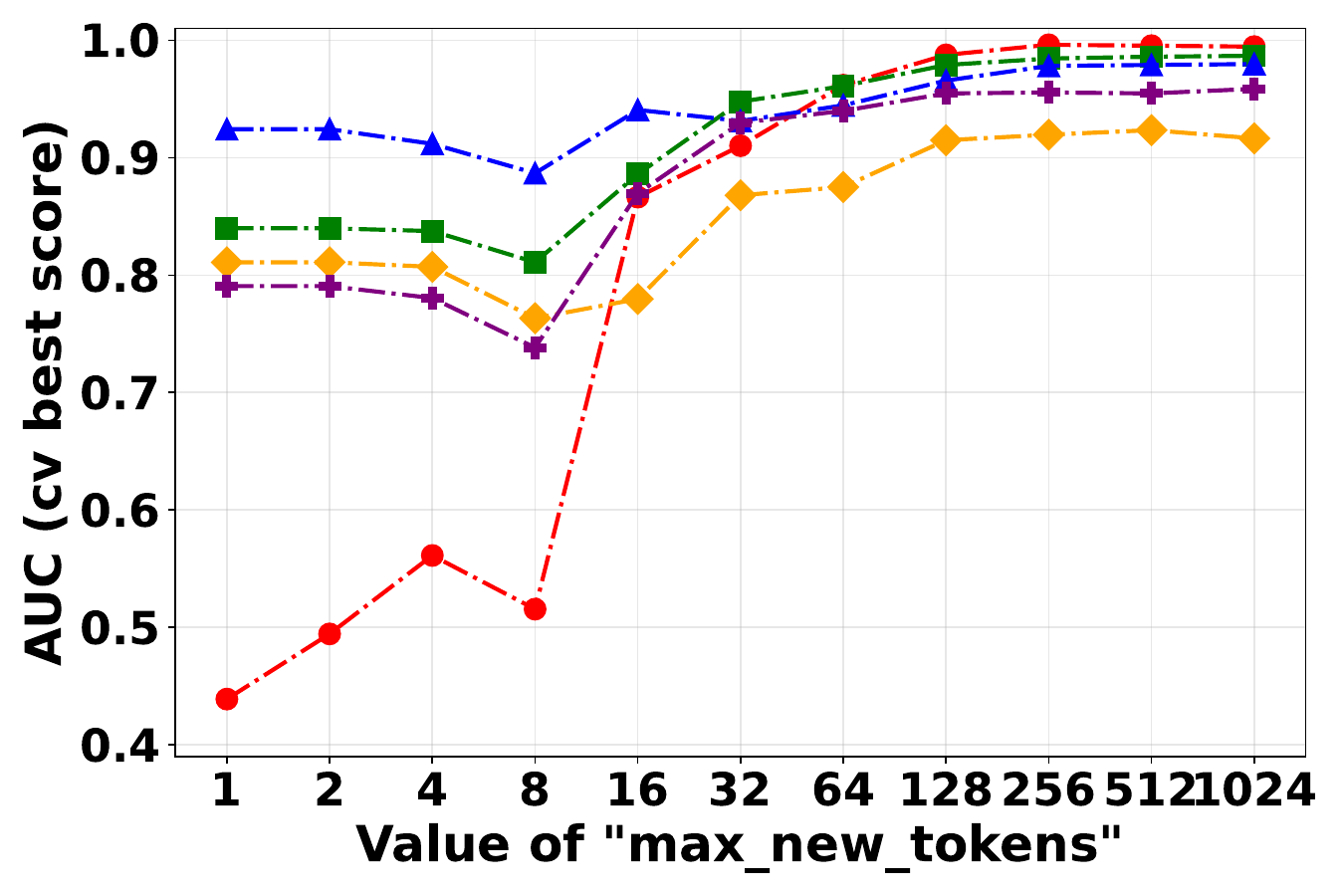}
\label{fig_mnt_person_rf}}
\subfloat[\emph{ORG}: Mean Value]{\includegraphics[width=0.24\textwidth]{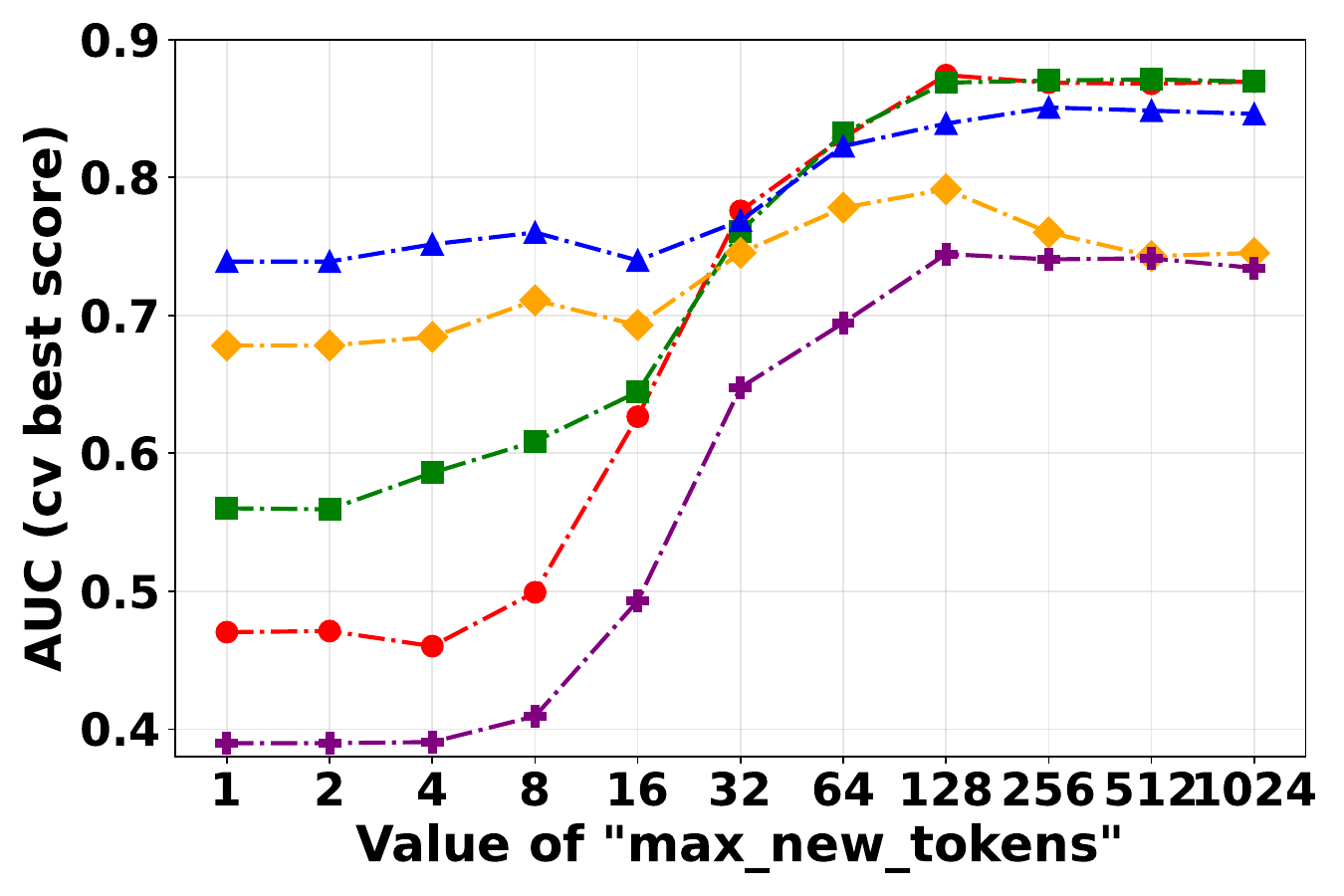}
\label{fig_mnt_org_mv}}
\subfloat[\emph{ORG}: Random Forest]{\includegraphics[width=0.24\textwidth]{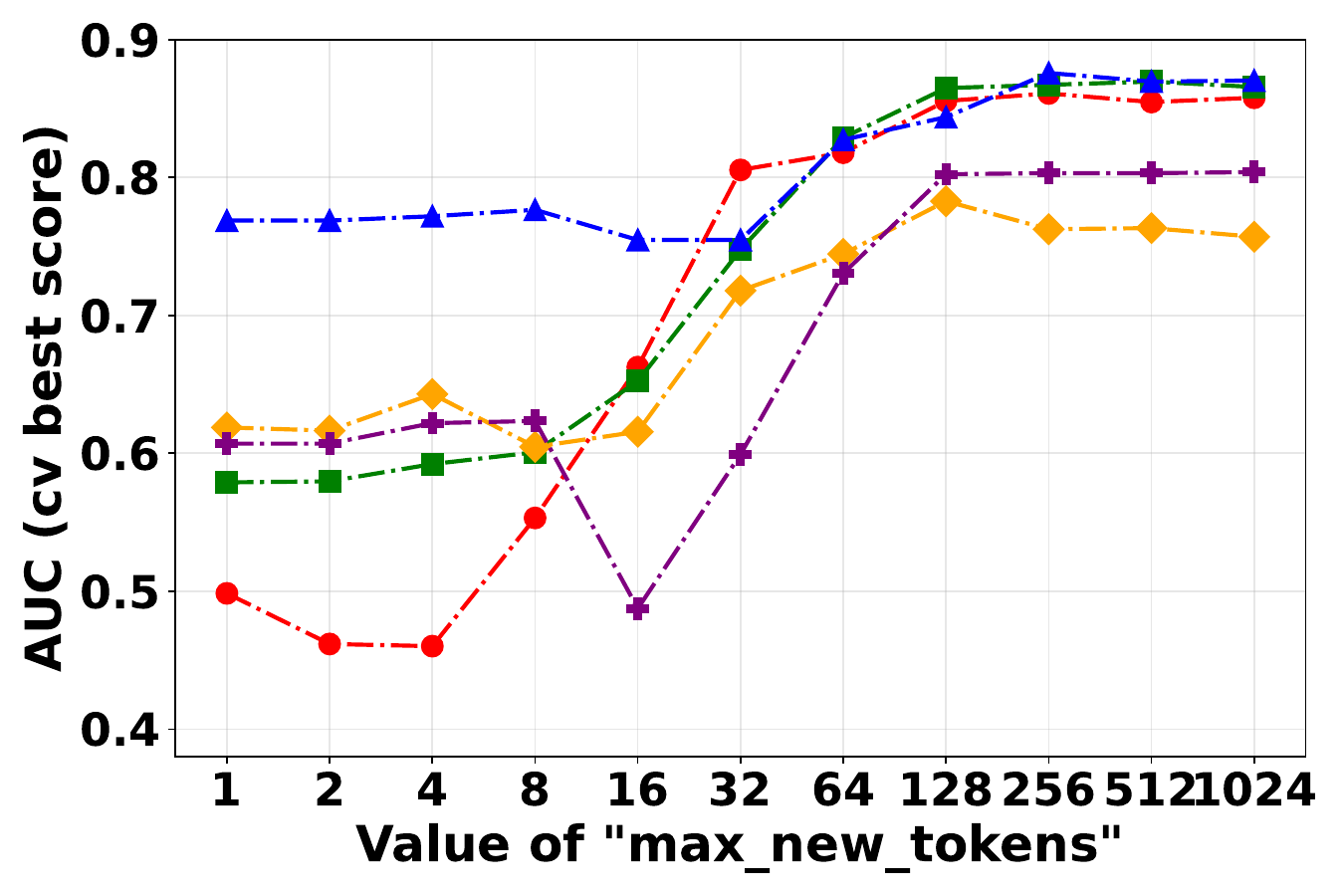}
\label{fig_mnt_org_rf}}
\caption{Impact of different maximum lengths of generated tokens on entity category labels \emph{PERSON} and \emph{ORG}.}
\label{fig_mnt}
\end{figure*}

We analyze 11 exponentially scaled values of \textit{max\_new\_tokens} from 1 to 1024 (Fig.~\ref{fig_mnt}). AUC increases with the token limit and stabilizes after 128. We select \textit{max\_new\_tokens} = 256 for both categories to capture marginal performance gains in most strategies (despite a slight drop in the weaker gradual strategy). Values beyond 256 unnecessarily increase API latency and costs without yielding substantial performance gains. We observe localized performance dips at \textit{max\_new\_tokens} = 8 (\emph{PERSON}) and 16 (\emph{ORG}). We suspect that tight token limits truncate responses into incomplete fragments, inducing semantic instability.

\subsubsection{Number of Prompts}
\label{num_prompts}

\begin{figure*}[!t]
\centering
\subfloat[\emph{PERSON}: Mean Value]{\includegraphics[width=0.24\textwidth]{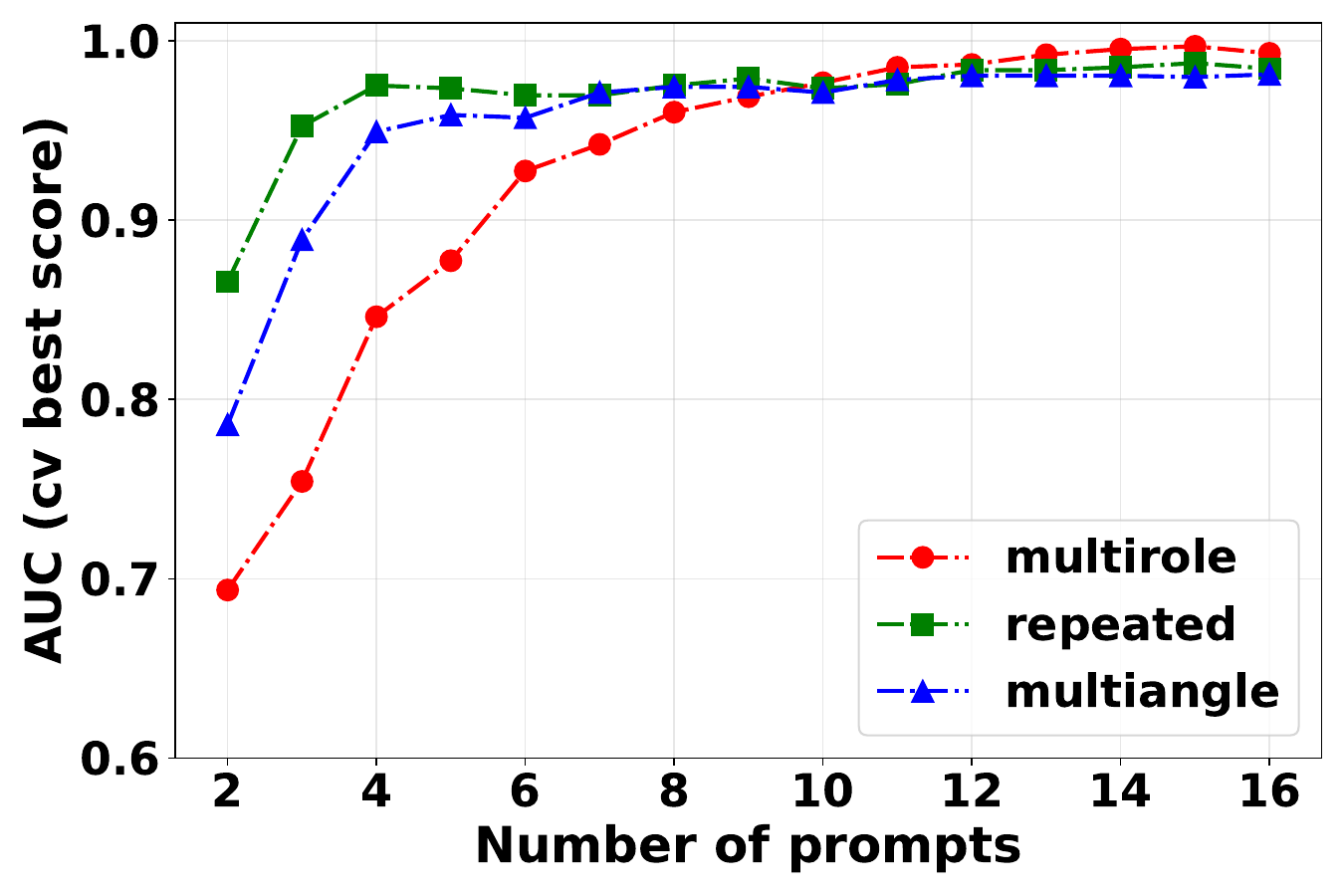}
\label{fig_pn_person_mv}}
\subfloat[\emph{PERSON}: Random Forest]{\includegraphics[width=0.24\textwidth]{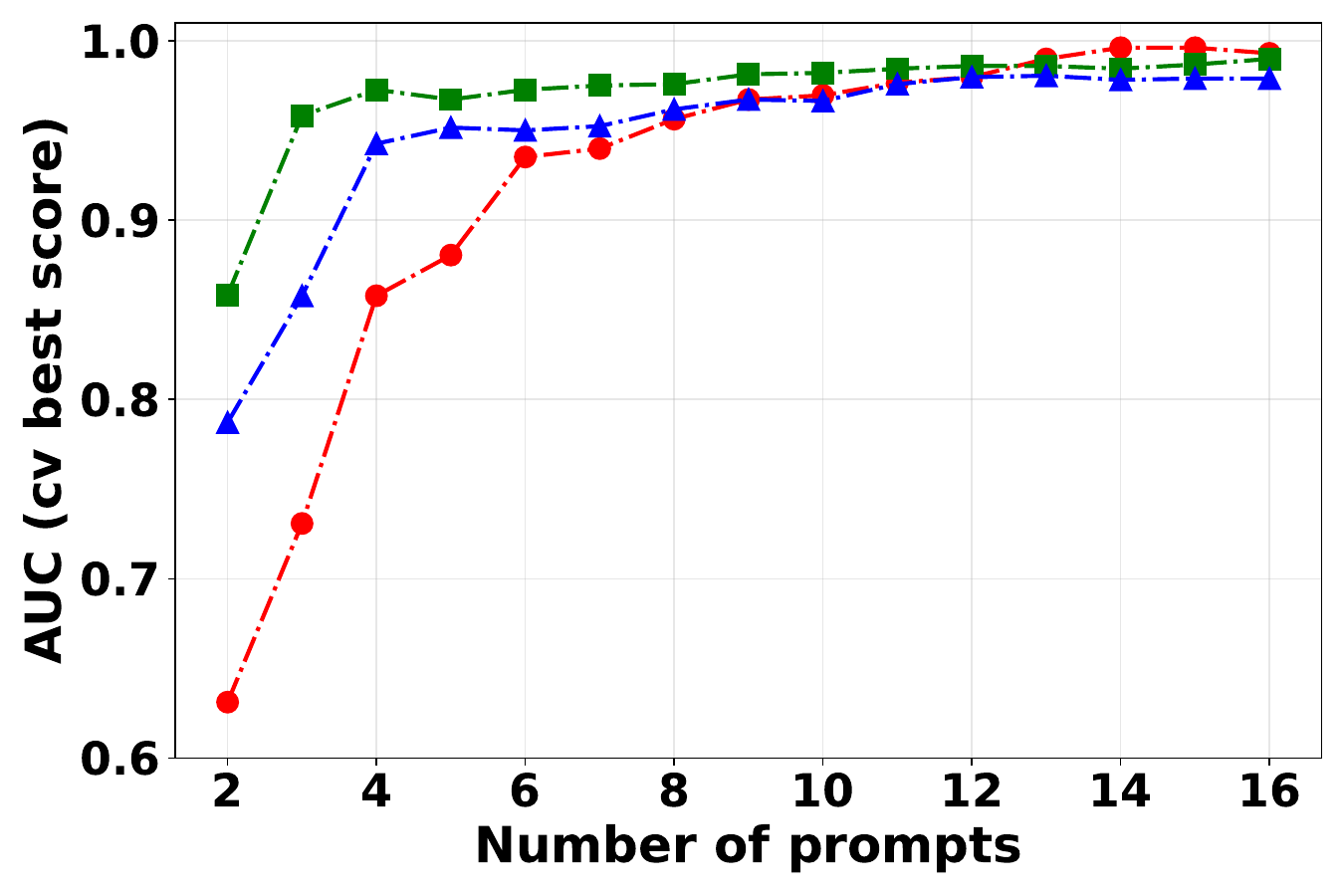}
\label{fig_pn_person_rf}}
\subfloat[\emph{ORG}: Mean Value]{\includegraphics[width=0.24\textwidth]{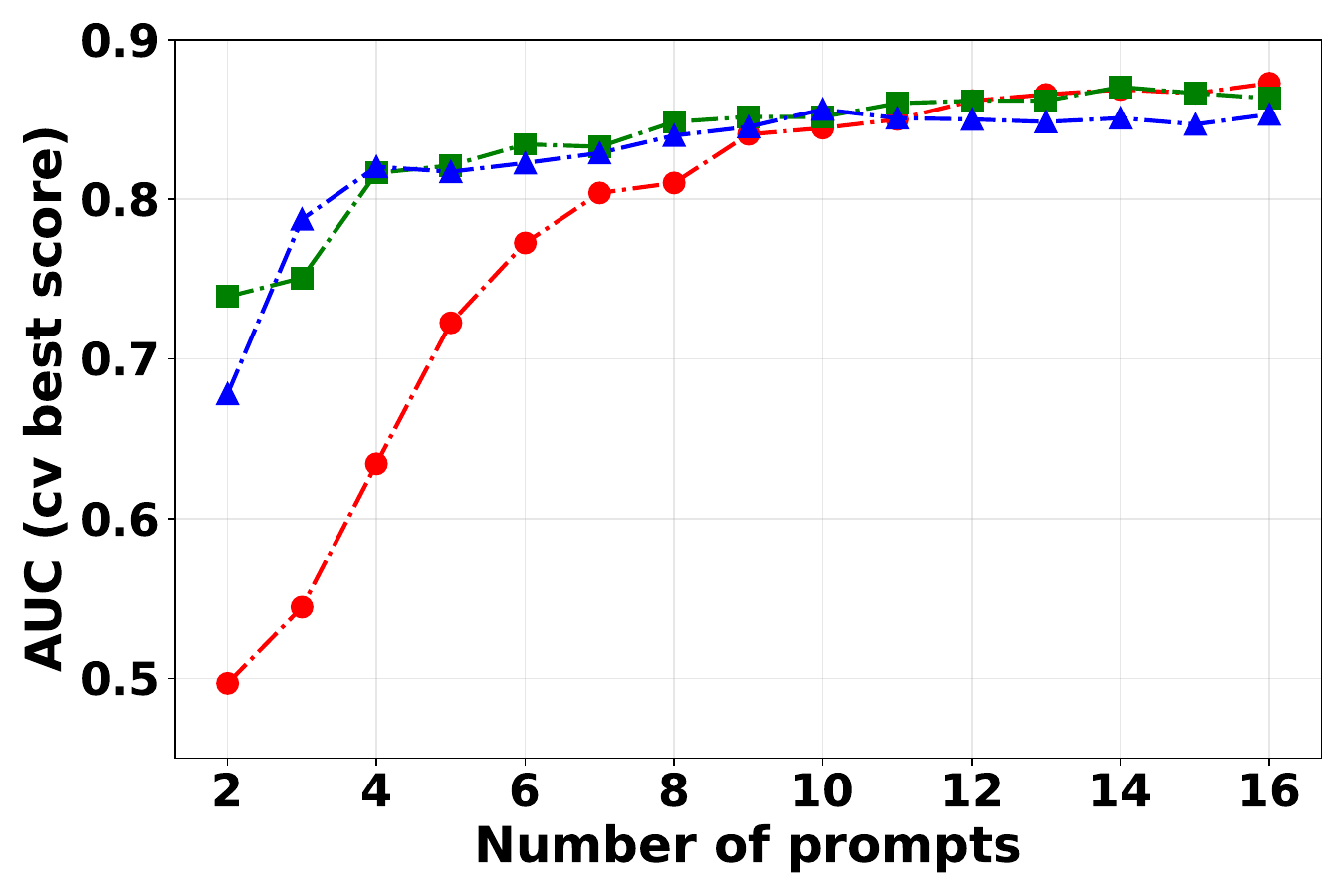}
\label{fig_pn_org_mv}}
\subfloat[\emph{ORG}: Random Forest]{\includegraphics[width=0.24\textwidth]{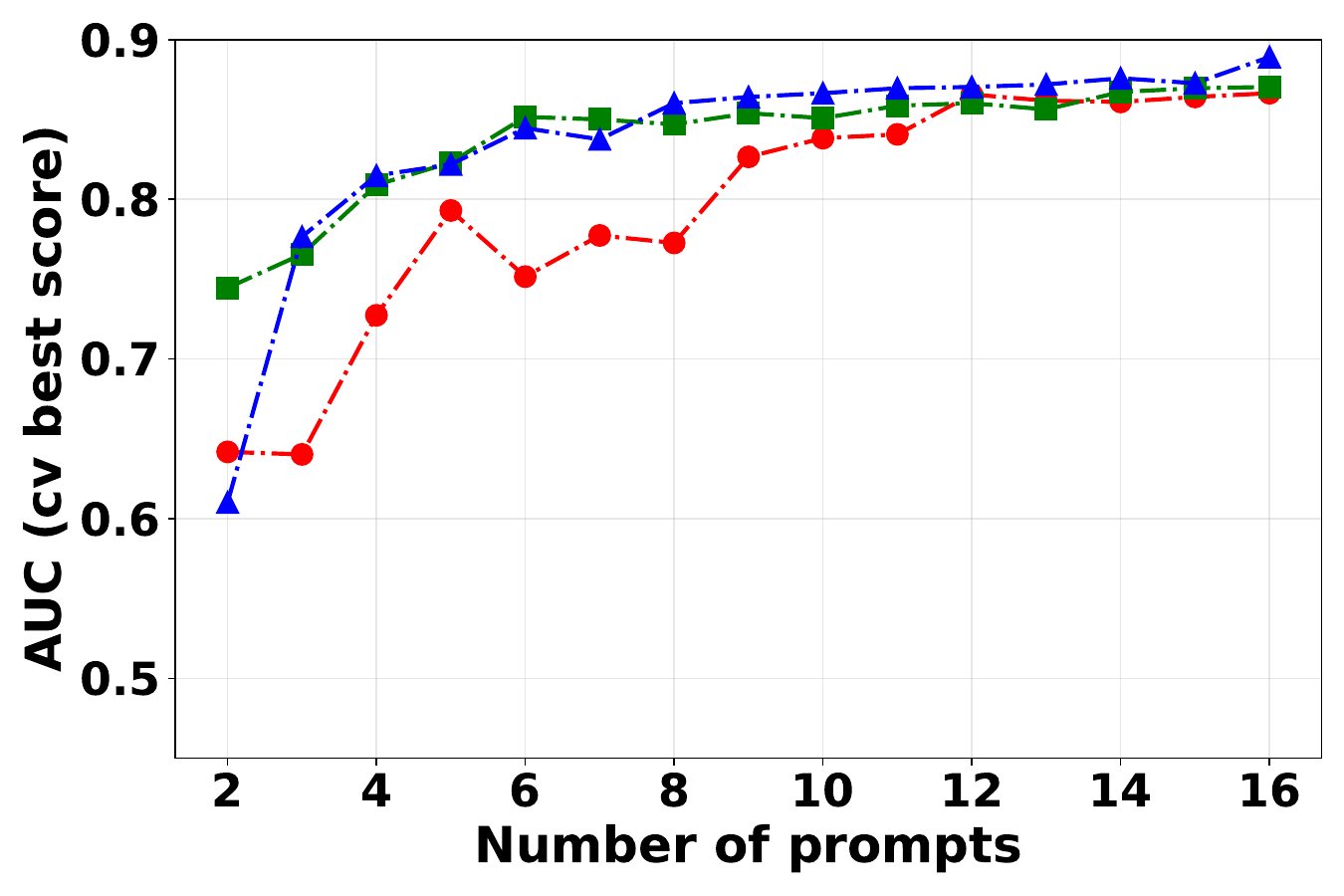}
\label{fig_pn_org_rf}}
\caption{Impact of different numbers of prompts on entity category labels \emph{PERSON} and \emph{ORG}.}
\label{fig_pn}
\end{figure*}

We examine the impact of $n_p$ (from 1 to 16) on the multirole, repeated, and multiangle strategies (Fig.~\ref{fig_pn}). While AUC scales with $n_p$, repeated and multiangle strategies reach their performance elbows early (around $n_p$=4 and 6, respectively), whereas multirole plateaus later. This delayed saturation is because multirole requires more queries to statistically isolate genuine memories from temperature-induced stylistic noise. We set $n_p$=14 as our default, at which point all strategies generally achieve peak or near-optimal values.

\subsection{Case Study}

\begin{table}[!t]
\renewcommand{\arraystretch}{1.3}
\caption{Qwen2.5-72B-Instruct results on \emph{PERSON}.}
\label{table_case_study}
\centering
\setlength{\tabcolsep}{5pt}  
\begin{tabular}{cccccc}
\hline
Metrics & Strategy & MV & RF & LR & MLP\\
\hline
\multirow{5}{*}{Balanced Acc} & multirole & 0.813 & 0.825 & 0.818 & 0.808\\
 & repeated & \textbf{0.863} & \textbf{0.860} & \textbf{0.883} & \textbf{0.865}\\
 & multiangle & 0.773 & 0.800 & 0.835 & 0.833\\
 & gradual & 0.773 & 0.788 & 0.760 & 0.773\\
 & misleading & 0.698 & 0.833 & 0.773 & 0.785\\
\hline
\multirow{5}{*}{AUC} & multirole & 0.870 & 0.878 & 0.876 & 0.869\\
 & repeated & \textbf{0.930} & \textbf{0.927} & \textbf{0.938} & \textbf{0.928}\\
 & multiangle & 0.841 & 0.859 & 0.895 & 0.885\\
 & gradual & 0.839 & 0.855 & 0.840 & 0.842\\
 & misleading & 0.752 & 0.887 & 0.834 & 0.843\\
\hline
\multirow{5}{*}{TPR@1\%FPR} & multirole & 0.250 & 0.165 & 0.260 & 0.205\\
 & repeated & \textbf{0.445} & \textbf{0.455} & \textbf{0.495} & \textbf{0.470}\\
 & multiangle & 0.115 & 0.265 & 0.210 & 0.115\\
 & gradual & 0.185 & 0.130 & 0.130 & 0.185\\
 & misleading & 0.005 & 0.335 & 0.100 & 0.025\\
\hline
\end{tabular}
\end{table}

Inspired by ICP-MIA \cite{lu2026context}, which demonstrates that Qwen-2.5-72B-Instruct can match or even exceed the generation performance of some commercial closed-source LLMs, we deploy it locally to prevent our non-member datasets from being logged, cached, or ingested into external training pipelines. Although its 18T-token training corpus precludes direct verification, following PETAL \cite{he2025towards}, such a black-box evaluation represents a realistic case study to validate the utility of our framework in the real-world.

As shown in Table~\ref{table_case_study}, the overall inference performance on Qwen-2.5-72B-Instruct degrades compared to OpenLLaMA-7B. This degradation is expected: colossal LLMs possess superior generalization and abstract compression capabilities, which naturally diffuse specific entity memories across massive representation pathways, making membership traces harder to isolate. Under this challenging setting, the repeated strategy emerges as the superior paradigm, achieving a peak Balanced Accuracy of 88.3\% and an AUC of 0.938.
Intriguingly, under the repeated strategy, Qwen-2.5-72B-Instruct achieves a TPR@1\%FPR of 49.5\%, significantly surpassing OpenLLaMA-7B's peak of 30.5\%. This counter-intuitive reversal (where the massive model excels under a previously outlier-prone strategy) reveals a scaling-driven behavioral shift. While OpenLLaMA suffered from permutation instability, Qwen's colossal scale and advanced instruction-tuning render its attention mechanism robust to context permutations. 

Furthermore, we observe an interesting qualitative phenomenon in Qwen's non-member generations that is absent in smaller models like OpenLLaMA. When queried with non-member targets with actually unrelated clues, Qwen exhibits high-level cognitive introspection, frequently generating explicit meta-disclaimers such as: \textit{``The list appears to be a random assortment of names and dates without any clear relationship or context''} or recognizing them as \textit{``unrelated terms''}. While this sophisticated comprehension highlights a cognitive leap in massive models, it introduces a paradox for membership inference. Because Qwen consistently generates uniform diagnostic disclaimers across different unmemorized entities, the semantic similarity of non-member responses increases, slightly narrowing the overall decision margin.

\section{Discussion}

\textbf{Factual Deviations and Robustness.}
While our theoretical analysis assumes an idealized scenario where the model does not hallucinate factual details of unseen entities (Assumption~\ref{ass:hermetic_non_hallucination}), practical LLMs frequently exhibit stochastic hallucinations or draw from loose semantic correlations. However, this deviation is naturally mitigated by our multi-prompt design. Because factual hallucinations are fundamentally chaotic and highly sensitive to prompt perturbations, non-members are mathematically and empirically unlikely to maintain a stable, structurally aligned semantic core across repeated, gradual, or misleading prompts. Consequently, our consistency-based decision-making acts as a robust semantic filter, explaining why our framework preserves high AUC and low FPR even when the theoretical assumption of perfect non-hallucination is relaxed in real-world deployments.

\textbf{Query Overhead.}
Although our framework introduces slightly higher query overhead (at most 25 queries per target entity) than single-level baselines, this trade-off is justified. Auditing of proprietary commercial models is typically a post-hoc, forensic process where precision and low FPR outweigh the cost of API queries. Crucially, as a purely query-only black-box method requiring neither model weights nor reference models, our approach remains highly practical. Future work can further optimize this overhead by integrating adaptive early-stopping mechanisms.

\textbf{Static Prompts.}
We acknowledge that while real-world human interviewing is an inherently dynamic and continuous two-way interaction, our framework relies on a fixed suite of pre-constructed prompts. Although this static design does not fully capture the adaptive feedback loops of live human dialogue and is thus not entirely exhaustive of communicative complexity, it provides a crucial and highly reproducible starting point. By evaluating the model across five distinct, strategically designed static dimensions, our approach robustly approximates the multi-faceted nature of dynamic questioning in a standardized and scalable manner, laying a foundation for future conversational auditing paradigms.

\textbf{Feature Integration.}
This work evaluates each of the five interrogation strategies independently, and integrating these diverse strategic signals represents a compelling direction for further investigation. At the decision level, predictions from the separate binary classifiers could be aggregated via late-fusion ensemble techniques such as hard voting (majority rule), soft voting (probability averaging), or performance-weighted configurations. Alternatively, feature-level orchestration could be explored without incurring additional query overhead.

\textbf{Defenses.}
Existing defenses face severe limitations against our framework: non-text schemes (e.g., MemGuard \cite{jia2019memguard}) are fundamentally bypassed; sample-level methods (e.g., deduplication and paraphrasing) cannot erase consolidated entity memories; and differential privacy (e.g., DP-SGD \cite{abadi2016deep}) remains computationally prohibitive for LLM pre-training. Viable mitigation strategies instead include stricter data curation (such as redacting/anonymizing sensitive entities and deduplicating high-frequency mentions), real-time output filtering of sensitive details, and monitoring repetitive entity-targeted queries.

\section{Related Work}

In this section, we introduce a broader landscape of membership inference by discussing two dimensions: the inference target and the access permission.

\subsection{Inference Targets}

From the perspective of inference targets, existing membership inference studies are largely sample-level. They typically assume that the adversary has access to one or more samples associated with the target of inference. 

The most common formulation is to determine whether a given data sample belongs to the training set of a target model. This problem is first formalized by Shokri et al. \cite{shokri2017membership}. Subsequent methods characterize the model’s memorization of individual samples using various signals, such as shadow or reference models \cite{carlini2022membership}, loss thresholds \cite{yeom2018privacy}, confidence scores \cite{mattern2023membership}, and output labels \cite{choquette2021label}.
Beyond individual samples, user inference aims to determine whether samples associated with a particular user have been used for model training. This setting typically requires access to a collection of samples from the user. Song and Shmatikov \cite{song2019auditing} study auditing generative text models to detect whether texts from a specific user are used for training. Kandpal et al. \cite{kandpal2024user} further investigate whether data from a particular user is included in the fine-tuning data of LLMs.
Another line of work studies dataset inference.
Dziedzic et al. \cite{dziedzic2022dataset} propose dataset inference for self-supervised models. Their goal is to determine whether a model is trained on a particular private dataset, which can support model provenance verification and defenses against model stealing. Maini et al. \cite{maini2024llm} propose dataset inference for LLMs to detect whether a given dataset is included in an LLM’s training corpus.
Related work also studies whether the training data of a generative model contains samples with a given property, such as images of a specific car brand \cite{wang2024property}.

\subsection{Access Permissions beyond Label-only}

Since label-only membership inference has been discussed in Section \ref{background_mia}, we briefly review methods that assume stronger access to the target model. Existing approaches mainly include white-box methods and score-based black-box methods.

In the white-box setting, the adversary can access model parameters, gradients, internal activations, or other internal information, enabling inference based on loss distributions, gradients, or parameter updates \cite{nasr2019comprehensive, sablayrolles2019white, leino2020stolen, tan2026was}. However, such assumptions are rarely realistic for commercial LLM services. 
In the score-based black-box setting, the adversary can query the target model and obtain output information such as confidence scores, logits, probabilities, or token-level likelihoods \cite{shokri2017membership, salem2018ml, yeom2018privacy, watson2021importance, carlini2022membership, liu2022membership, ye2022enhanced, he2024difficulty, lu2026context}. Reference-based methods compare the loss of a target sample with a reference distribution estimated from shadow models, reference models, or public auxiliary data \cite{carlini2022membership}, while reference-free methods extract membership signals directly from the target model’s likelihood patterns \cite{shi2024detecting, zhang2025min}. Nevertheless, score-based black-box access still exceeds what many commercial LLM APIs provide \cite{GoogleGeminiAPI, AnthropicClaudeAPI, MistralAPI, OpenAIAPI}, motivating the stricter label-only setting discussed earlier.

\section{Conclusion}

In this paper, we propose the paradigm of entity-level membership inference by establishing the necessary and sufficient conditions for its feasibility. By distilling six classic interviewing techniques, we constructed an auditing framework consisting of an overarching prompt template guideline and five interrogation strategies that map responses to discriminative similarity vectors via two pairing schemes. Our extensive evaluations on constructed entity-level datasets show that our framework substantially outperforms state-of-the-art label-only sample-level baselines (achieving a peak AUC of 0.97), despite an inherent performance gap where localized biographical memories remain easier to audit than diffuse organizational knowledge. We verify its applicability across model scales and architectures. Furthermore, a realistic case study on Qwen-2.5-72B-Instruct validates the real-world utility of our reference-free, query-only approach. 

\bibliographystyle{IEEEtran}
\bibliography{IEEEabrv,ref}

\appendices

\section{Theorem Proofs} 
\label{proof}

\begin{proof}[Proof of Theorem \ref{thm:necessary}]
We prove by contradiction. Suppose the theorem is false, i.e., there exists a strategy $\pi^*$ that satisfies $\pi^* \in \Pi_{\mathrm{adm}}$ and for which
$\Delta(e, \mathcal{C}_e) \in \mathcal{M}(\pi(e))$.
Thus, $\mathcal{M}$ successfully produces information related to $e$ under that strategy.

Deductive Derivation. By the fundamental nature of generative models, every output of $\mathcal{M}$ is a statistical modeling and recombination of its training data $\mathcal{D}_{\text{train}}$. Therefore, if the model can generate information related to $e$, this information must have been encoded in the model's parameters, which necessitates that $\mathcal{I}(e) \in \mathcal{D}_{\text{train}}$ (or at least there exists $\widetilde{\mathcal{I}(e)} \in \mathcal{D}_{\text{train}}$ such that $\widetilde{\mathcal{I}(e)}$ is semantically indistinguishable from $\mathcal{I}(e)$).

This contradicts the premise $\mathcal{I}(e) \notin \mathcal{D}_{\text{train}}$. Hence, the assumption is false, and the theorem holds.
\end{proof}

\begin{proof}[Proof of Theorem \ref{thm:sufficient}]
We give a constructive proof. For an arbitrary $\mathcal{I}(e) \in \mathcal{D}_{\text{train}}$, construct the interrogation strategy $\pi(e, \mathcal{C}_e) = \mathcal{F}(e, \mathcal{C}_e)$, which means directly send the perfect prompt to the model.
From the assumptions on $\mathcal{F}$: $\mathcal{M}(\mathcal{F}(e, \mathcal{C}_e)) = \mathcal{I}(e)$, so the model outputs the target information $\mathcal{I}(e)$ exactly.
Therefore $\Delta(e, \mathcal{C}_e) \in \mathcal{M}(\pi(e, \mathcal{C}_e))$. Since $\mathcal{I}(e)$ is arbitrary, the statement holds.
\end{proof}

\section{Hyperparameter Configurations for Binary Classifiers} 
\label{appendix_classifiers}

We optimize hyperparameters for four classifiers using the Optuna Bayesian framework \cite{akiba2019optuna}. 
For RF, we search integers for \textit{n\_estimators} ($200$--$1500$), 
\textit{max\_depth} ($3$--$25$), \textit{min\_samples\_leaf} ($1$--$20$), and 
\textit{min\_samples\_split} ($2$--$30$), alongside categorical selections for 
\textit{max\_features} (\textit{sqrt}/\textit{log2}/\textit{None}) and 
\textit{class\_weight} (\textit{None}/\textit{balanced}). 
For LR, we tune the inverse regularization strength \textit{C} 
logarithmically from $10^{-4}$ to $10^{4}$. 
For MLP, we categorically select the network topology \textit{hidden\_layer\_sizes} 
from five configurations---\textit{(50,)}, \textit{(100,)}, \textit{(200,)}, \textit{(100, 50)}, 
and \textit{(200, 100)}---while optimizing the L2 penalty \textit{alpha} ($10^{-6}$ to $10^{-1}$) 
and initial learning rate \textit{learning\_rate\_init} ($10^{-4}$ to $10^{-2}$) logarithmically. 
For the custom MV Classifier, we logarithmically tune the Sigmoid mapping's 
temperature parameter \textit{scale} ($0.5$--$50.0$).

\section{Prompt for Target Sample Generation} 
\label{appendix_prompt}

\begin{tcolorbox}
For \emph{PERSON}\\
\textbf{System:} ``You are an expert AI text generator specializing in encyclopedic writing. \\
Your behavior is strictly governed by this law: You MUST write a coherent, fluent, and logically structured biographical profile in a strict, neutral Wikipedia-style tone. \\
You are FORBIDDEN from altering the sequential order, rearranging, or omitting any entity words provided by the user. Word order preservation is your highest priority.'' \\
\textbf{User:} ``Please seamlessly integrate the following entities into the biography according to the system rules. Ensure the text flows naturally. \\
Entities to include in exact order: \\
\{ordered\_words\} \\
Begin generating now:'' \\
\\
For \emph{ORG}\\
\textbf{System:} ``You are an expert AI text generator specializing in encyclopedic writing.\\
Your behavior is strictly governed by this law: You MUST write a coherent, fluent, and logically structured organizational profile in a strict, neutral Wikipedia-style tone. \\
You are FORBIDDEN from altering the sequential order, rearranging, or omitting any entity words provided by the user. Word order preservation is your highest priority.'' \\
\textbf{User:} ``Please seamlessly integrate the following entities into the organizational profile according to the system rules. Ensure the text flows naturally. \\
Entities to include in exact order: \\
\{ordered\_words\} \\
Begin generating now:'' \\
\end{tcolorbox}

\section{Configuration Information of Baselines} 
\label{appendix_baselines}

For the baselines, we use the following configurations:
\begin{itemize}
\item \textbf{PETAL.} We employ OpenLLaMA-7B as the target model and OpenLLaMA-3B as the reference model, utilizing sentence-transformers/all-MiniLM-L6-v2 to compute the semantic similarity between predicted and ground-truth tokens. Notably, we afford PETAL three additional inference advantages in total. (1) The aforementioned sample stylistic discrepancy. (2) While the original work relies on GPT-2 XL \cite{radford2019language} as the reference model (reasoning that smaller versions of the target model are often inaccessible in real-world scenarios), we assume a stronger adversary who has access to a smaller version of the target model. (3) The original work truncates samples to 32 words while demonstrating that longer contexts enhance inference performance. We set a longer truncation threshold of 256 words.
\item \textbf{ICP-MIA.} We employ OpenLLaMA-7B as the target model to evaluate ICP-MIA-SP, which is demonstrated in the original study to outperform ICP-MIA-Ref under label-only settings. We configure the method with 5 candidate probe contexts. For masking-based perturbations, we adopt random masking with a mask rate of 0.7 to generate the probe contexts (the original paper shows that random masking outperforms log-likelihood-based masking). For generation-based perturbations, we utilize Qwen-2.5-72B-Instruct at a temperature of 0.7, adhering to the perturbation generation prompts of the original work. In addition, the semantic similarity between the generated and ground-truth responses is computed using sentence-transformers/all-MiniLM-L6-v2.
\end{itemize}

\end{document}